\documentclass{article}

\usepackage{arxiv}

\usepackage[round]{natbib} %
\usepackage{mathtools} %
\usepackage{amssymb}
\usepackage{bm}
\usepackage{bbm}
\usepackage{amsfonts}       %
\usepackage{algorithm}
\usepackage[noend]{algpseudocode}
\usepackage{booktabs} %
\usepackage{tikz} %
\usetikzlibrary{automata, positioning, arrows} %
\usetikzlibrary{matrix}
\usepackage{minitoc} %

\usepackage[utf8]{inputenc} %
\usepackage[T1]{fontenc}    %
\usepackage[capitalize]{cleveref}
\usepackage{url}            %
\usepackage{booktabs}       %
\usepackage{amsfonts}       %
\usepackage{nicefrac}       %
\usepackage{microtype}      %
\usepackage{xcolor}         %
\usepackage{pgfplots}
\usepackage{subfig}

\usepackage{enumitem}

\newcommand{\StateSpace}{\mathcal{S}}
\newcommand{\Action}{\mathcal{A}}
\newcommand{\StateAction}{\StateSpace \times \Action}

\newcommand{\Transition}{\mathcal{P}}
\newcommand{\Horizon}{H}
\newcommand{\reward}{r}
\newcommand{\G}{\mathcal{G}}
\newcommand{\M}{\mathcal{M}}
\newcommand{\V}{V}
\renewcommand{\P}{\mathbb{P}}
\newcommand{\mPhi}{\Phi}

\newcommand{\abs}[1]{\left| {#1} \right|}
\newcommand{\paren}[1]{\left( {#1} \right)}
\newcommand{\bracket}[1]{\left[ {#1} \right]}
\newcommand{\curly}[1]{\left\{ {#1}\right\}}
\renewcommand{\O}{\mathcal{O}}
\newcommand{\E}{\mathbb{E}}
\renewcommand{\L}{\mathcal{L}}

\newcommand{\F}{\mathcal{F}}
\newcommand{\Simplex}[1]{\Delta_{#1}}
\renewcommand{\epsilon}{\varepsilon}
\newcommand{\vlambda}{\bm{\lambda}}

\newcommand{\vpi}{\bm{\pi}}

\newcommand{\wtilde}[1]{\widetilde{#1}}
\newcommand{\Reals}{\mathbb{R}}
\newcommand{\Alg}{\mathfrak{A}}
\usepackage{amsthm}
\usepackage{thmtools}

\newtheorem{lemma}{Lemma}
\newtheorem{assumption}{Assumption}

\newtheorem{definition}{Definition}
\newtheorem{corollary}{Corollary}
\newtheorem{example}{Example}

\newcommand{\AlgName}{\textbf{C}oordinate-\textbf{A}scent for \textbf{CMPG}s}
\newcommand{\AlgNameShort}{CA-CMPG}
\newcommand{\Dope}{\textsc{Dope}}

\usepackage[colorinlistoftodos, textwidth=18mm]{todonotes}

\newcommand{\AlgNameExp}{\textbf{C}oordinate-\textbf{A}scent for \textbf{CMPGs} with \textbf{E}xploration}
\newcommand{\AlgNameShortExp}{CA-CMPG-E}

\title{Provably Learning Nash Policies in \\Constrained  Markov Potential Games}

\author{%
  Pragnya Alatur \\
  Department of Computer Science\\
  ETH Zurich and ETH AI Center\\
  \texttt{pragnya.alatur@ai.ethz.ch} \\
   \And
  Giorgia Ramponi \\
  Department of Computer Science\\
  ETH Zurich and ETH AI Center\\
  \texttt{giorgia.ramponi@ai.ethz.ch} \\
   \AND
   Niao He \\
  Department of Computer Science\\
  ETH Zurich\\
  \texttt{niao.he@inf.ethz.ch} \\
  \And
  Andreas Krause \\
  Department of Computer Science\\
  ETH Zurich\\
  \texttt{krausea@ethz.ch} \\
}

\begin{document}
\maketitle
\date{}

\begin{abstract}
Multi-agent reinforcement learning (MARL) addresses sequential decision-making problems with multiple agents, where each agent optimizes its own objective. In many real-world instances, the agents may not only want to optimize their objectives, but also ensure safe behavior. For example, in traffic routing, each car (agent) aims to reach its destination quickly (objective) while avoiding collisions (safety). Constrained Markov Games (CMGs) are a natural formalism for safe MARL problems, though generally intractable. In this work, we introduce and study {\em Constrained Markov Potential Games} (CMPGs), an important class of CMGs. We first show that a Nash policy for CMPGs can be found via constrained optimization. One tempting approach is to solve it by Lagrangian-based primal-dual methods. 
As we show, in contrast to the single-agent setting, however, CMPGs do not satisfy strong duality, rendering such approaches inapplicable and potentially unsafe. To solve the CMPG problem, we propose our algorithm \AlgName~(\AlgNameShort), which provably converges to a Nash policy in tabular, finite-horizon CMPGs. Furthermore, we provide the first sample complexity bounds for learning Nash policies in unknown CMPGs, and, which under additional assumptions, guarantee safe exploration.
\end{abstract}

\vspace{-1mm}
\section{Introduction}
\vspace{-1mm}
\looseness -1
Multi-Agent Reinforcement Learning (MARL) addresses sequential decision-making problems with {\em multiple agents}, where the decisions of individual agents may also affect others. In this work, we focus on a rich and fundamental class of MARL problems, known as {\em Markov Potential Games}, \citep[MPGs,][]{leonardos2022global}. Important applications, such as traffic routing \citep{altman2006survey} or wireless communication \citep{yamamoto2015network}, can be modeled as MPGs. The main characteristic of an MPG is the existence of an underlying {\em potential function}, which captures the agents' incentives to deviate between different policies. MPGs can model both fully cooperative scenarios\footnote{In fully cooperative scenarios, the agents have one common objective.} and scenarios, in which the agents have individual objectives, as long as such a potential function exists.

\looseness -1
In many real-world applications, however, the standard MPG framework fails to incorporate additional requirements like {\em safety}. For instance, in traffic routing, we do want to find the fastest route to the individual destinations, while ensuring that the vehicles drive safely and do not collide. In this work, we introduce the framework of {\em Constrained Markov Potential Games} (CMPGs) to study safety in the context of MPGs. We incorporate safety using {\em coupled} constraints on the policies of the agents. Coupled constraints are relevant because they allow us to model requirements like collision avoidance. Our objective is to find a Nash policy \citep{nash1950equilibrium, altman2000constrained}, i.e., a set of policies such that no agent has the incentive to deviate unilaterally within the constrained set of policies. Prior work on algorithms for (unconstrained) MPGs, in which each agent improves its own objective {\em independently}, cannot be applied to the constrained setting, as the agents may need to {\em coordinate} to satisfy the constraints. A more detailed discussion of prior work is provided in \cref{sec:related-work}.
We study tabular CMPGs in the finite-horizon setting and summarize our contributions here:
\begin{enumerate}[topsep=0.pt,parsep=1.5pt,partopsep=1.5pt,leftmargin=*]
    \item \looseness -1 First, we show that a Nash policy can in principle be recovered by solving a constrained optimization problem, which, however, becomes intractable as the number of agents increases (\cref{sec:cmpg-duality}).
    \item \looseness -1 Given tractable algorithms for unconstrained MPGs \citep[cf.][]{leonardos2022global, fox2022independent}, a tempting approach would be to utilize Lagrangian duality to reduce the constrained problem to an unconstrained one \citep{Diddigi2019, Parnika2021}.
Unfortunately, we show that strong duality does not hold for our problem (\cref{sec:cmpg-duality}), rendering such approaches {\em sub-optimal} and {\em unsafe}.
    This is in sharp contrast to the single-agent setting, for which strong duality does hold \citep{paternain2019constrained}.
    \item Instead of solving the constrained optimization problem, we propose to directly search for a Nash policy. We present our algorithm -- \AlgName~(\AlgNameShort) -- which provably converges to an $\epsilon$-Nash policy, assuming that the agents have full knowledge of the CMPG (\cref{sec:solving-cmpgs}).
    \item \looseness -1 Finally, we prove a sample complexity bound for our algorithm \AlgNameShort, when the agents do not know the CMPG beforehand (\cref{sec:learning-cmpgs}). With access to a generative model (\cref{subsec:cmdp-generative-model}), the agents converge to an $\epsilon$-Nash policy with $\wtilde{\O}{\paren{\frac{H^8}{\epsilon^3 \zeta^2}}}$ samples, where $\zeta$ is the Slater constant of the CMPG and $H$ is the horizon. On the other hand, if the agents do not have access to a generative model, but still want to ensure safe exploration, we obtain a sample complexity bound of $\wtilde{\O}{\paren{\frac{H^{10}}{\epsilon^5 c^2}}}$ (\cref{subsec:cmdp-no-regret}), where $c\in (0,\zeta]$ is a quantity related to the constraint set of the CMPG.
\end{enumerate}

\vspace{-1mm}
\section{Related Work}
\vspace{-1mm}
\label{sec:related-work}

\looseness -1\textbf{Markov Potential Games:}
MPGs have become popular in recent years and have been studied for the tabular setting \citep{leonardos2022global, zhang2022softmax, zhang2021gradient, chen2022poa, mao2022decentralized, Maheshwari2022, fox2022independent} and for state-action spaces with function approximation \citep{dongsheng2022mpgconvergence}. For the tabular setting with {\em known} rewards and transitions, \citet{leonardos2022global} prove that independent policy gradient (IPG) converges to an $\epsilon$-Nash policy in $O(1/\epsilon^2)$ iterations. If rewards and transitions are {\em unknown}, \citet{mao2022decentralized} prove that IPG with access to a stochastic gradient oracle converges to an $\epsilon$-Nash policy with a {sample complexity} of $\O\paren{1/\epsilon^{4.5}}$.
In these IPG algorithms, the agents improve their own objectives {\em independently}. It is challenging to apply these algorithms with coupled constraints, as the agents may need to coordinate to satisfy those constraints, at least during the learning process. \citet{Song2021} present a different approach for tabular MPGs with unknown rewards and transitions, in which the agents {\em coordinate} to compute an  $\epsilon$-Nash policy with a sample complexity of $\wtilde{\O}(1/\epsilon^3)$. While their algorithm is for unconstrained MPGs, we show in our work, that this type of approach can be extended to the constrained setting. \citet{Maheshwari2022} present a different approach with asymptotic convergence to a Nash policy, whereas we target finite-time convergence. Note that MPGs are only one way to model MARL problems, and for a more comprehensive overview on MARL, we refer the reader to the surveys by \citet{yang2021} and  \citet{zhang2021multi}.

\looseness -1\textbf{Constrained Markov Decision Processes:} A common approach to constrained {\em single-agent} RL are 
\looseness -1 \textit{Constrained Markov Decision Processes} \citep[CMDPs,][]{altman1999constrained}. CMDPs are widely studied, and a comprehensive survey is given by \citet{garcia2015}. Below, we focus on aspects relevant to our work.
In CMDPs, the agent optimizes a reward function subject to constraints. Lagrangian duality is a common approach for constrained optimization and \citet{paternain2019constrained} proved that CMDPs possess the {\em strong duality property}, giving theoretical justification for the use of Lagrangian dual approaches.

\looseness -1 \textbf{Constrained Markov Games:} One of the common approaches to constrained multi-agent RL are {\em Constrained Markov Games} \citep[CMGs,][]{altman2000constrained}. CMGs restrict the policies of the agents, which can be used to model safety objectives. Note that CMPGs are one class of CMGs.
In cooperative CMPGs\footnote{Note that cooperative games are a strict subclass of CMPGs, as CMPGs are able to model non-cooperative settings too.}, where the agents have one common reward function, the CMPG objective very much resembles the CMDP formulation. Furthermore, \citet{Diddigi2019} and \citet{Parnika2021} demonstrate good experimental results for cooperative CMPGs with Lagrangian dual approaches, but provide no theoretical guarantees. We prove in our work, however, that strong duality does not hold in general for CMPGs (cf. \cref{sec:cmpg-duality}), rendering Lagrangian dual approaches inapplicable in those cases. Furthermore, we demonstrate that the dual might even return unsafe solutions.

\vspace{-1mm}
\section{Background and Problem Definition}
\vspace{-1mm}
\label{sec:background}

\looseness -1 \textbf{Notation:} For any $n\in\mathbb{N}$, we use the short-hand notation $[n]$ to refer to the set of integers $\curly{1,...,n}$. For any finite set $X$, we denote by $\Simplex{X}$ the probability simplex over $X$, i.e., $\Simplex{X} = \{v \in [0,1]^{|X|} | \sum_{x \in X} v(x) = 1\}$.

\subsection{Markov Potential Games}
\looseness -1 An $n$-agent {\em Markov Potential Game} (MPG) is a tuple $\mathcal{G} $ $=$ $(\StateSpace, \{\Action_i\}_{i=1}^n,\Horizon, \curly{\Transition_h}_{h=1}^{\Horizon},$ $\{\curly{\reward_{i,h}}_{h=1}^{\Horizon}\}_{i=1}^n, \mu )$, where $\StateSpace$ is the state space, $\Action_i$ is agent $i$'s action space. We denote by $\Action \triangleq \times_{i=1}^n \Action_i$ the joint action space, $\Horizon \in \mathbb{N}_{>0}$ the horizon. $\Transition_h:\StateAction\rightarrow\Simplex{\StateSpace}$ is the environment's transition function at time $h\in[\Horizon]$ and $\Transition_h(s'|s,a)$ denotes the probability of moving to state $s'$ from state-action pair $(s,a)\in\StateAction$ at step $h\in[\Horizon]$,  $r_{i,h}:\StateAction \rightarrow [0,1]$ is agent $i$'s reward function at step $h\in[\Horizon]$ and $\mu\in\Simplex{\StateSpace}$ denotes the initial state distribution. We assume $\StateSpace$ and $\Action$ to be finite.

\textbf{Policies:}
For every agent $i\in [n]$, we define its policy space as $\Pi^i\triangleq \curly{\{\pi_{i,h}\}_{h=1}^H \;|\;\pi_{i,h}:\StateSpace\rightarrow\Simplex{\Action_i}, \forall h\in [H]}$. If agent $i$ follows a policy $\pi\in\Pi^i$, it means that at step $h\in [H]$ and state $s\in \StateSpace$, the agent samples its next action from $\pi_{h}(\cdot|s)$.
We denote by $\Pi\triangleq \curly{ \vpi = \paren{\pi_1,...,\pi_n} | \pi_i \in \Pi^i, \forall i\in[n]}$ the set of {\em joint} policies. For any policy $\vpi\in \Pi$ and agent $i\in[n]$, we denote by $\vpi_{-i}$ the policy of the {\em other} $n-1$ agents.

\textbf{Value Function:}
For any policy $\vpi\in\Pi$ and agent $i\in[n]$, the value function $\V^{r_i}(\vpi)$ measures the expected, cumulative reward of agent $i$, and is defined as follows:
\begin{equation}
\label{eq:value-function}
    \V^{r_i}(\vpi)\triangleq\mathop{\E}_{\substack{s \sim \mu,\\ a_h \sim \vpi_h(\cdot | s_h),  \\s_{h+1} \sim \Transition_h(\cdot | s_h, a_h) }}\big[\sum_{h=1}^{\Horizon} r_{i,h}(s_h,a_h) | s_0 = s\big].
\end{equation}

\textbf{Potential Function:} \looseness -1 An MPG possesses an underlying potential function $\mPhi:\Pi\rightarrow \Reals$ such that:
\begin{equation}
\begin{gathered}
\label{eq:potential-property}
\V^{r_i}(\pi_i, \vpi_{-i}) - \V^{r_i}(\pi'_i, \vpi_{-i}) = \mPhi(\pi_i, \vpi_{-i}) - \mPhi(\pi'_i, \vpi_{-i}) \qquad
\forall \pi'_i \in \Pi^i, \forall \vpi \in \Pi, \forall i \in [n].
\end{gathered}
\end{equation}
This is an adaptation of the potential function defined in \citet{leonardos2022global} to the finite-horizon setting. Instead of defining a per-state potential function, we directly consider the potential function with respect to the initial distribution $\mu$. 

\looseness -1 \textbf{Remark:} Note that the potential function is a property of the MPG and is typically not known to the agents. In a cooperative game, the agents have one shared reward function $r$ such that $r_i \equiv r$, $\forall i \in [n]$. In this case, the potential function is simply the value function of the agents, i.e., $\mPhi = \V^r$. Note, however, that cooperative games are a {\em strict} subset of MPGs, and MPGs have the ability to express non-cooperative scenarios, such as traffic congestion. In \cref{sec:experiments}, we describe different instances in detail.

\subsection{Constrained Markov Potential Games}
\looseness -1 An $n$-agent {\em Constrained Markov Potential Game} (CMPG) is an MPG $\mathcal{G}$ $= (\StateSpace, \curly{\Action_i}_{i=1}^n, $ $\Horizon, \curly{\Transition_h}_{h=1}^{\Horizon}, \{\curly{\reward_{i,h}}_{h=1}^{\Horizon}\}_{i=1}^n, \mu)$ with constraints $\{( \curly{ c_{j,h}}_{h=1}^{\Horizon}, \alpha_j) \}_{j=1}^k$, where $c_{j,h}:\StateAction\rightarrow[0,1]$ denotes the $j$-th cost function at step $h\in [\Horizon]$ and $\alpha_j\in [0,H]$ is the constraint threshold.\footnote{Even though we define our problem in the finite-horizon setting, our results can be easily extended to the discounted, infinite-horizon setting.}

\textbf{Feasible Policies:} We call a policy $\vpi \in \Pi\,$ {\em feasible}, if it satisfies the following constraints:
\begin{equation*}
    V^{c_j}_{\mu}(\vpi)\triangleq \mathop{\E}_{\substack{s \sim \mu,\\ a_h \sim \vpi_h(\cdot | s_h),  \\s_{h+1} \sim {\Transition_h}(\cdot | s_h, a_h) }}\big[\sum_{h=1}^{\Horizon} c_{j,h}(s_h,a_h) \Big| s_0 = s \big] \leq \alpha_j, \quad
    \forall j \in [k].
\end{equation*}
\looseness -1
In the rest of the paper, we use $\Pi_C$ to refer to the set of {\em feasible }policies. For every agent $i$ and policy $\vpi_{-i}$ of the other $n-1$ agents, we define $\Pi_C^i(\vpi_{-i}) \triangleq \curly{ \pi_i\in \Pi^i | (\pi_i, \vpi_{-i}) \in \Pi_C }$. We refer to this type of constraints as {\em coupled} constraints, as the values of the constraints depend on the {\em joint} actions of the agents. If we wish to model an intersection in a traffic scenario, an important constraint to incorporate would be collision avoidance. To decide whether a certain set of actions causes a collision or not, we need to take the actions of {\em all} agents at the intersection into account.

In a CMPG, each agent $i$ aims to maximize its own value function $\V^{r_i}$. Since the rewards and transitions depend on the {\em joint} policy, it may not be possible to find a policy that is globally optimal for all value functions simultaneously. Instead, the agents typically need to settle for an equilibrium policy, at which no agent has an incentive to deviate unilaterally. Many different types of equilibria exist in the literature, such as the Nash equilibrium \citep{nash1950equilibrium}, correlated equilibrium \citep{aumann1987correlated} or Stackelberg equilibrium \citep{breton1988sequential}. In this work, our goal is to obtain a {\em Nash equilibrium policy} \citep{nash1950equilibrium, altman2000constrained} in a CMPG. We define a relaxed notion in the following paragraph.

\textbf{$\varepsilon$-Nash Equilibrium Policy:} For any $\epsilon\geq 0$, a policy $\vpi^*=(\pi_1^*,...,\pi_n^*)\in\Pi_C$ is a {\em $\epsilon$-Nash equilibrium policy}, if it is the $\epsilon$-best-response policy for each agent, i.e.,\footnote{This is an extension of the {\em generalized Nash equilibrium} \citep{facchinei2010generalized} to CMPGs.}:
\begin{equation}
\label{eq:nash_policy}
\begin{gathered}
\max_{\pi_i \in \Pi^i_C(\vpi^*_{-i})}\V^{r_i}(\pi_i, \vpi_{-i}^*) - \V^{r_i}(\vpi^*) \leq \epsilon, \qquad \forall i\in[n].
\end{gathered}
\end{equation}
We call $\vpi^*$ a {\em Nash equilibrium policy}, if \cref{eq:nash_policy} holds with $\epsilon=0$. In the rest of the paper, we refer to the Nash equilibrium policy as {\em Nash policy}.

\subsection{Constrained Markov Decision Processes}
\looseness -1
A {\em Constrained Markov Decision Process} (CMDP) is a tuple $\M =(\StateSpace, \Action, \Horizon, \curly{\Transition_h}_{h=1}^{\Horizon}, \{\reward_{h}\}_{h=1}^{\Horizon}, $ $ \mu, \{ ( \curly{ c_{j,h}}_{h=1}^{\Horizon}, \alpha_j) \}_{j=1}^k)$. In a CMDP, there is a {\em single} agent. However, the individual elements in $\M$ carry the same meaning as in CMPGs.
Furthermore, the policy sets $\Pi,\Pi_C$ and the value functions $\V^r:\Pi\rightarrow \Reals$ (reward), $\V^{c_j}:\Pi \rightarrow \Reals, j\in [k]$ (costs) are defined in the same way as for CMPGs. In a CMDP, the agent aims to find a policy $\pi^*$, that satisfies:
\begin{equation}
\label{eq:cmdp-objective}
\pi^* \in \arg\max_{\pi\in\Pi_C} \V^r(\pi).
\end{equation}
In the following section, we prove that a Nash policy in a CMPG can be found by maximizing the potential function with respect to the given constraints, similar to \cref{eq:cmdp-objective}. We will show that Lagrangian duality, a common approach for constrained optimization, will not work in general for CMPGs.

\vspace{-1mm}
\section{Duality for Constrained Markov Potential Games?}
\vspace{-1mm}
\label{sec:cmpg-duality}
\looseness -1
For an MPG with potential function $\mPhi$, a globally optimal policy $\vpi^*\in \arg\max_{\vpi\in\Pi}\mPhi(\vpi)$ is also a Nash policy \citep{leonardos2022global}. We show in \cref{lem:constr-global-max-nash} that this property  generalizes to CMPGs. We defer the proofs for the theoretical results in this section to \cref{app:cmpg-duality}.
\begin{restatable}{proposition}{lemconstrglobalmaxnash}
\label{lem:constr-global-max-nash}
    Define the following constrained optimization problem:
    \begin{equation}
        \label{eq:cmpg-optimization-problem}
        \begin{gathered}
        \vpi^* \in \arg\max_{\vpi\in\Pi_C} \mPhi(\vpi).
    \end{gathered}
    \end{equation}
    Then, $\vpi^*$ is a Nash policy for a CMPG with potential function $\mPhi$.
\end{restatable}
Solving \cref{eq:cmpg-optimization-problem} directly is not trivial; even if the agents know the rewards and transitions, the potential function is usually not known. Moreover, the fact that we have {\em coupled} constraints makes solving \cref{eq:cmpg-optimization-problem} directly intractable.
Nevertheless, a common approach for solving constrained optimization problems is {\em Lagrangian duality}, which, in our case, turns the CMPG into an (unconstrained) MPG with modified rewards (\cref{lem:dual-modified-mpg}).  This would enable the use of scalable algorithms that have been developed for unconstrained MPGs \citep{leonardos2022global}.
Furthermore, in previous works \citep{liu2021cmix,Diddigi2019}, Lagrangian duality was used for cooperative CMPGs and showed promising experimental results. 
This makes Lagrangian duality a tempting approach for CMPGs.
For this, we define the {\em Lagrangian} $\L:\Pi\times\Reals_+^k\rightarrow \Reals$ and the primal\footnote{Note that the primal is equivalent to \cref{eq:cmpg-optimization-problem}.} and dual problems for \cref{eq:cmpg-optimization-problem} as follows:
\begin{align}
    \L(\vpi, \vlambda) &\triangleq \mPhi(\vpi) + \sum_{j=1}^k \lambda_j \paren{\alpha_j - \V^{c_j}(\vpi)} \tag{Lagrangian}\label{eq:lagrangian} \\
    \tag{Primal} \label{eq:primal}
P^* &= \max_{\vpi\in\Pi} \min_{\vlambda\in\Reals_+^k} \L(\vpi, \vlambda)\\
\tag{Dual} \label{eq:dual}
D^* &= \min_{\vlambda\in\Reals_+^k} \max_{\vpi\in\Pi}  \L(\vpi, \vlambda).
\end{align}
As a first step, in \cref{lem:dual-modified-mpg}, we prove that the dual problem does indeed correspond to an (unconstrained) MPG.
\begin{restatable}{proposition}{lemdualmodifiedmpg}
\label{lem:dual-modified-mpg}
For any $\vlambda \in \Reals_+^k$, $\L(\cdot, \vlambda)$ is a potential function for an MPG with reward functions $\tilde{r}_{i,h} \triangleq r_{i,h} - \sum_{j=1}^k \lambda_j c_{j,h}, \forall i\in [n], \forall h\in [\Horizon]$.
\end{restatable}
Then, weak duality guarantees that $D^*\geq P^*$ holds. Unfortunately, in the following proposition, however, we show that {\em strong duality}, i.e., $D^*=P^*$,  {\em does not hold} in general for CMPGs. 

\begin{restatable}{proposition}{thmcmpgstrongduality}
\label{thm:strong-duality-cmpgs}
    There exists a CMPG, for which strong duality does not hold, i.e., for which $P^* \neq D^*$.
\end{restatable}
\begin{proof}
We prove this using a counter-example. Consider the following two-agent CMPG with $|\StateSpace|=1$, $\Action_1 = \Action_2 = \curly{1,2}$, reward functions $r = r_1 = r_2$, constraint function $c$ and threshold $\alpha=1/2$. The rewards and constraints are specified via the matrices
\begin{align*}
    A = \begin{bmatrix}
3 & 2\\
2 & 4 \end{bmatrix}, \quad B=\begin{bmatrix}
0 & 0\\
0 & 1 \end{bmatrix},
\end{align*}

where $r(i,j) = A(i,j)$ and $c(i,j)=B(i,j), \forall i,j\in \{1,2\}$. This is a \textit{cooperative} CMPG with potential function $\mPhi(\vpi) = \pi_1^T A \pi_2$. The optimization formulation (\cref{eq:cmpg-optimization-problem}) corresponding to this CMPG is:
\begin{equation}
\label{eq:counter-example}
\max_{\pi_1,\pi_2\in\Delta_2} \pi_1^T A \pi_2,
    \text{ subject to: } \pi_1^T B \pi_2 \leq 1/2.
\end{equation}

\textbf{Primal problem:}
First, we solve the primal problem, which is defined as follows:
\begin{equation*}
    P^* = \max_{\pi_1,\pi_2\in\Delta_2} \min_{\lambda\in\Reals_+} \paren{\pi_1^T A \pi_2  + \lambda \paren{\frac{1}{2} - \pi_1^T B \pi_2}}
\end{equation*}
The policies $\pi_1 = \pi_2 = \bracket{1-\sqrt{\frac{1}{2}}, \sqrt{\frac{1}{2}}}$ solve the primal problem with a reward of $P^* \approx 3.09$.
One can easily verify that these three policies are also Nash policies.

\textbf{Dual problem:} Next, we solve the dual problem, which is defined as follows:
\begin{equation*}
    D^* =  \min_{\lambda\in\Reals_+} \underbrace{\max_{\pi_1,\pi_2\in\Delta_2} \paren{\pi_1^T A \pi_2  + \lambda \paren{ \frac{1}{2} - \pi_1^T B \pi_2}}}_{=:d(\lambda)},
\end{equation*}
where $d(\lambda)$ is the {\em dual function}. \cref{fig:dual-counterexample} visualizes $d(\lambda)$ for $\lambda\in [0,2]$.
\begin{figure}
    \centering
    \includegraphics[width=0.5\textwidth]{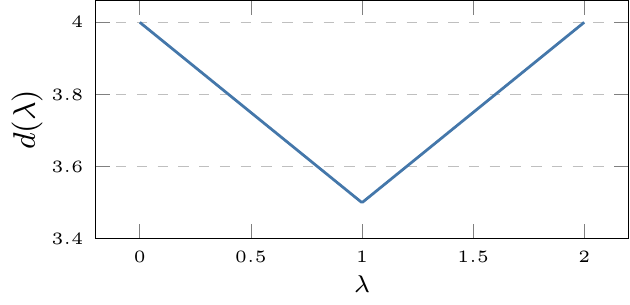}
    \caption{This figure displays the dual function $d(\lambda)$ for the CMPG in \cref{eq:counter-example},  evaluated at 1000 equidistant locations $\lambda\in[0,2]$.}
    \label{fig:dual-counterexample}
\end{figure}
Since the dual function is always convex, it is sufficient to focus only on this interval. From \cref{fig:dual-counterexample}, we can see that $d(\lambda)$ reaches its minimum at $\lambda_D^* = 1$ with $D^* = d(\lambda_D^*)=3.5$. Note that this is strictly larger than the primal solution $P^* = 3$, and therefore, {strong duality does not hold here.} Next, we list two policies that are solutions to the dual problem, i.e., policies $\pi_D^*$ that satisfy $\pi_D^*\in \arg\max_{\pi} \curly{\pi_1^T A \pi_2 + \lambda_D^* \paren{\frac{1}{2} - \pi_1^T B \pi_2}}$:
\begin{enumerate}
    \item $\pi_1 = \bracket{1, 0}, \pi_2 = \bracket{1, 0}$
    \item $\pi_1 = \bracket{0, 1}, \pi_2 = \bracket{0, 1}$
\end{enumerate}
The first policy satisfies the constraints and is indeed a Nash policy, with a reward of $3<P^*$. The second policy, however, does not satisfy the constraints. Solving the dual problem does therefore not necessarily guarantee a feasible policy.

\end{proof}

\looseness -1\textbf{Remark:} To give an intuition on \cref{thm:strong-duality-cmpgs}, consider a cooperative CMPG with $\mPhi \equiv \V^r$, i.e., the potential function is equal to the shared value function $\V^r$. Note that, in this case, the primal problem very much resembles the CMDP objective (\cref{eq:cmpg-optimization-problem}) and it is tempting to solve the CMPG as a CMDP with a large action space $\Action = \times_{i=1}^n \Action_i$. Recall also, that strong duality does indeed hold for CMDPs \citep{paternain2019constrained} and CMDPs can be solved via primal-dual algorithms. By solving this large CMDP, we obtain a solution $\vpi^*$ that specifies distributions over the {\em joint} action space $\Action$. To obtain a solution for the original CMPG, however, we require a policy that can be factored into a set of independent policies $\curly{\pi^*_i}_{i\in [n]}$ such that $\vpi^*_h(a|s) = \prod_{i=1}^n \pi^*_{i,h}(a_i|s), \forall (s,a,h)\in \StateAction\times[H]$.

\vspace{-1mm}
\section{Solving Constrained Markov Potential Games}
\vspace{-1mm}
\label{sec:solving-cmpgs}
\looseness -1 In this section, we propose an efficient algorithm to compute Nash policies in CMPGs\footnote{Note that we may not find a Nash policy that solves \cref{eq:cmpg-optimization-problem} though.}. Similar to the work on unconstrained MPGs by \citet{Song2021}, in our algorithm \AlgName~(\AlgNameShort), agents take turns to solve a {\em Constrained Markov Decision Process} (CMDP), i.e., a single-agent reinforcement learning problem, in every iteration. To do this, the agents need to coordinate, such that, when one agent is solving the CMDP, the others provide a stationary environment to that agent by keeping their policies fixed. There are some technical challenges compared to the unconstrained MPG setting. The main difference is that in the CMPG setting, to ensure the convergence to a Nash policy, we need also to ensure that the intermediate policies remain {\em feasible} (see remark at the end of this section). Our algorithm \AlgNameShort~is described in \cref{alg:cmpg-ca-no-exploration}.

 \begin{algorithm}[t]
    \centering
\caption{\AlgNameShort~(Known Transitions)}\label{alg:cmpg-ca-no-exploration}
    \begin{algorithmic}[1]
\Require $\epsilon>0$ (approximation error), $\vpi^S\in\Pi_C$ (feasible policy), $T$ (number of iterations)
\State $\vpi^0 \gets \vpi^S$
\For{$t=1,...,T$}
    \For{agent $i = 1,...,n$}
        \State Agent $i$ computes $\hat{\pi_i}^t$ such that \cref{eq:cmpg-ca-cmdp} is satisfied. \label{algln:cmpg-ca-no-exploration-cmdp}
        \State $\epsilon_i^t \gets {\V^{r_i}}({\hat{\pi}}_i^t, \vpi_{-i}^{t-1}) - {\V^{r_i}}(\vpi^{t-1})$.
    \EndFor
    \If{$\max_{i\in[n]}\epsilon_i^t > \epsilon/2$}
        \State Set $\vpi^t = (\hat{\pi}_j^t, \vpi_{-j}^{t-1})$, where $j=\arg\max_{i\in[n]} \epsilon_i^t$, break ties arbitrarily.
    \Else
        \State \textbf{break}
    \EndIf
\EndFor
\end{algorithmic}
\end{algorithm}
We assume for now that the agents know their own reward functions, the cost functions as well as the transition model. As a starting point for \AlgNameShort, the agents require access to a feasible, initial policy, which we state in the following assumption:
\begin{assumption}
\label{assumption:cmpg-feasible-initial-policy}
    Given a CMPG, the agents have access to a feasible policy $\vpi^S \in \Pi_C$.
\end{assumption}
This type of assumption is common for safe exploration in CMDPs \citep{Bura2022, liu2021learning}. We discuss in \cref{app:solving-cmpgs}, why we require it for our setting. While, in general, it may be computationally hard to compute a feasible $\vpi^S$ in the multi-agent setting, we now discuss two examples, for which it is easy to compute $\vpi^S$.

\begin{example}[Single Constraint] Consider the problem $\min_{\vpi\in\Pi} \V^{c_1}(\vpi)$. Since the constraint set is feasible, we must have that $\min_{\vpi\in\Pi} \V^{c_1}(\vpi) \leq \alpha_1$. Note that this is an unconstrained Markov decision process (MDP) with state space $\StateSpace$ and action space $\Action$. It is well-known that MDPs always possess at least one {\em deterministic}, optimal policy, which can be computed using dynamic programming techniques. Thus, we compute a deterministic policy $\vpi^C\in\arg\min_{\vpi\in\Pi}\V^{c_1}(\vpi)$, s.t. for every state $s\in\StateSpace$ and step $h\in[H]$, there is exactly one action $a=(a_1,...,a_n)\in\Action$, for which $\vpi^C_h(a|s)=1$ and $\vpi^C_h(a'|s)=0, \forall a'\neq a$. Then, for every agent $i\in [n]$, we set $\pi^C_{i,h}(a_i|s)=1$ and $\pi^C_{i,h}(a'_i|s)=0$, for all $a'_i\neq a_i$. It is easy to verify that $\vpi^C = \prod_{i=1}^n \pi^C_i$.
\end{example}

\begin{example}[Independent Transitions and Composite Constraints] Consider a CMPG with per-agent state spaces $\StateSpace_1,...,\StateSpace_n$ and transition models $\Transition_1,...,\Transition_n$, where $\Transition_{j,h}(s'|s, a)$ is the probability that agent $j$ transitions to state $s'\in\StateSpace_j$ from state-action pair $(s,a)\in \StateSpace_j\times\Action_j$ at step $h\in [H]$. We denote by $\StateSpace\triangleq \times_{i=1}^n \StateSpace$ the joint state space and define $\Transition_h(s'|s,a) \triangleq \prod_{i=1}^n \Transition^i_h(s'_i|s_i,a_i)$ as the joint probability of transitioning to state $s'\in\StateSpace$ from state-action pair $(s,a)\in\StateAction$ at step $h\in [H]$. Furthermore, assume that for each $j\in[k]$, the constraint function $c_j$ can be written as $c_{j,h}(s,a) \triangleq \sum_{i=1}^n c_{j,h}^i(s_i,a_i)$. Due to this, the cumulative constraints can be written as $\V^{c_j}(\vpi) = \sum_{i=1}^n \V^{c_j^i}(\pi_i)$, $\forall j\in [k]$. To find a feasible policy, each agent $i\in[n]$ computes $\pi_{i} \in \curly{ \pi\in\Pi^i \Big| \V^{c_j^i}(\pi) \leq c_i^*, \forall j\in [k] }$, where $c_i^* \triangleq \min_{c\in\Reals}\curly{ \exists \pi\in \Pi^i \Big| \V^{c_j^i}(\pi) \leq c, \forall j\in [k] }$. Assuming that the constraint set is feasible, it is easy to see that $\vpi^S = (\pi^S_1,...,\pi^S_n)$ must be feasible.
\end{example}

In \AlgNameShort, the agents start with the feasible policy $\vpi^S$. In every iteration, the agents take turns to maximize their own value function. While one agent is maximizing its value function, the other agents keep their policy fixed (\cref{algln:cmpg-ca-no-exploration-cmdp}); therefore, that agent is essentially solving a CMDP. In iteration $t$, agent $i\in [n]$ faces the CMDP $\M = \paren{\StateSpace, \Action_i, \Horizon, \curly{\wtilde{\Transition}_h}_{h=1}^H, \curly{\wtilde{r}_h}_{h=1}^H, \mu, \curly{\paren{\curly{\wtilde{c}_{j,h}}_{h=1}^H, \alpha_j}}_{j=1}^k}$, where the reward function $\wtilde{r}$, cost functions $\curly{\wtilde{c}_j}_{j\in [n]}$ and transition model $\wtilde{\Transition}$ are defined according to \cref{eq:cmdp-reward-function}, \cref{eq:cmdp-constraint-function} and \cref{eq:cmdp-transitions}:
\begin{align}
    \wtilde{r}_h(s,a_i) &\triangleq \sum_{a_{-i} \in \Action\setminus \Action_i} r_{i,h}(s,(a_i,a_{-i})) \cdot \vpi_{-i,h}^{t-1}(a_{-i}|s), \label{eq:cmdp-reward-function}\\
    \wtilde{c}_{j,h}(s,a_i) &\triangleq \sum_{a_{-i} \in \Action\setminus \Action_i} c_{j,h}(s,(a_i,a_{-i})) \cdot \vpi_{-i,h}^{t-1}(a_{-i}|s), \label{eq:cmdp-constraint-function}\\
    \wtilde{\Transition}_h(s'|s,a_i) &\triangleq \sum_{a_{-i}\in \Action \setminus \Action_i} \Transition_h(s'|s, (a_i, a_{-i})) \cdot \vpi_{-i,h}^{t-1}(a_{-i} | s) \label{eq:cmdp-transitions},
\end{align}
for all $(s,a_i,s',h)\in \StateAction_i\times\StateSpace\times[H]$.
Let us recall the CMDP objective from \cref{eq:cmdp-objective}. In practice, we can only solve \cref{eq:cmdp-objective} {\em approximately}. Given $\epsilon>0$, we assume that in every iteration $t$, agent $i\in [n]$ can efficiently compute a policy $\hat{\pi}_i^t \in \Pi_C^i(\vpi_{-i}^{t-1})$ such that it satisfies the following conditions\footnote{This can be achieved using state-of-the-art primal-dual methods, such as the work by \citet{ding2020natural, paternain2019constrained}.}:
\begin{equation}
\label{eq:cmpg-ca-cmdp}
\begin{gathered}    
    \max_{\pi \in \Pi_C^i(\vpi_{-i}^{t-1})} \V^{r_i}(\pi, \vpi^{t-1}_{-i}) - \V^{r_i}(\hat{\pi}_i^t, \vpi_{-i}^{t-1}) \leq \epsilon/2.
\end{gathered}
\end{equation}
\looseness -1 Due to the potential property (\cref{eq:potential-property}), if agent $i\in[n]$ improves its own value function, it implicitly also improves the potential function. To prove that the potential function can be increased only a finite number of times, implying termination of \AlgNameShort, we require the potential function to be bounded.
\begin{restatable}{lemma}{lempotentialbounded}
\label{lem:potential-bounded}
    Fix an arbitrary base policy $\vpi^B \in \Pi$. Then, for every $\vpi\in\Pi$, the potential function can be bounded as: $\mPhi(\vpi)\leq nH + \mPhi(\vpi^B)$.
\end{restatable}
We defer the proofs to all theoretical results in this section to \cref{app:solving-cmpgs}. \AlgNameShort~terminates when the agents cannot deviate unilaterally and improve their value function by more than $\epsilon$, i.e., when they reach an $\epsilon$-Nash policy. We state this result in the following theorem:
\begin{restatable}{theorem}{thmconvergencecmpgnoexp}
\label{thm:convergence-cmpg-ca-no-exploration}
Suppose that Assumption \ref{assumption:cmpg-feasible-initial-policy} holds. Then, given $\epsilon>0$, if we invoke \AlgNameShort~
with $T=\frac{2 n H}{\epsilon}$, it converges to an $\epsilon$-Nash policy.
\end{restatable}

\looseness -1 \textbf{Remark:} What if we relax the feasibility requirement in \cref{eq:cmpg-ca-cmdp} and allow the CMDP solver to return an $\epsilon$-{\em feasible} policy $\vpi$ such that $\V^{c_j}(\vpi)\leq \alpha_j + \epsilon$, $\forall j\in [k]$, for an $\epsilon>0$? In that case, the intermediate policies might not be feasible and \AlgNameShort~may get stuck in an infeasible policy, which is not a Nash policy.

\vspace{-1mm}
\section{Learning in Unknown Constrained Markov Potential Games}
\vspace{-1mm}
\label{sec:learning-cmpgs}
\looseness -1 In this section, we assume that the agents do not know the transition model beforehand. For simplicity, we assume that they do know the rewards and costs\footnote{In general, learning the transitions is harder than learning rewards and costs. Concretely, this also means that learning rewards and costs will not add any dominating terms to the overall sample complexity (see \citet{vaswani2022nearoptimal}).}.
Our objective is to establish a {\em sample complexity} bound for learning in CMPGs. Concretely, we want to construct an algorithm, such that, given any $\epsilon>0, \delta\in (0,1)$, the algorithm returns an $\epsilon$-Nash policy with probability at least $1-\delta$, using at most $\F(\epsilon,\delta)$ {\em samples} from the transition model $\Transition$. Before we proceed, we define an important quantity related to the constraint set, which also contributes to the final sample complexity.
\begin{definition}[Slater constant]
\label{def:slater}
Given a feasible CMPG $\G$, we define its Slater constant $\zeta$ as follows:
\begin{equation*}
    \zeta\triangleq \min_{i\in[n]} \min_{\vpi_{-i}\in \Pi\setminus \Pi^i}  \max_{\pi\in\Pi^i} \{\alpha - \V^c(\pi, \vpi_{-i})\}.
\end{equation*}
We call $\G$ {\em strictly} feasible if and only if $\zeta>0$.
\end{definition}
\looseness -1In the rest of this section, we assume that the agents face an unknown, strictly feasible CMPG with Slater constant $\zeta>0$. Next, we discuss which parts of \AlgNameShort~need to be adapted for this setting.
\begin{enumerate}
[topsep=0.pt,parsep=1.5pt,partopsep=1.5pt,leftmargin=*]
    \item \looseness -1 In every iteration $t$, each agent $i\in [n]$ needs to solve the CMDP described in \cref{sec:solving-cmpgs} (\cref{algln:cmpg-ca-no-exploration-cmdp}). To solve this CMDP, we assume access to a {\em sample-efficient} CMDP solver, which has the following guarantees: Given $\epsilon>0, \delta \in (0,1)$, the solver uses at most $\F_{C}\paren{|\StateSpace|,|\Action_i|,\Horizon,\zeta,\delta,\frac{\epsilon}{4}}$ samples and returns a policy $\hat{\pi}_i^t\in\Pi_C^i(\vpi_{-i}^{t-1})$ such that it satisfies the following, with probability at least $1-\delta$:
    \begin{equation}
    \label{eq:sample-efficient-cmdp-solver}
        \max_{\pi\in \Pi_C^i(\vpi_{-i}^{t-1})} \V^{r_i}(\pi, \vpi_{-i}^{t-1}) - \V^{r_i}(\hat{\pi_i}^t, \vpi_{-i}^{t-1}) \leq \epsilon/4.
    \end{equation}
    \looseness -1Compared to the setting with known transitions, we have a stricter bound on the approximation error of $\epsilon/4$ here. We discuss in \cref{app:learning-cmpgs}, why we require this.

    \item \looseness -1 To compute $\epsilon_i^t$ in step $t$,  agent $i$ needs to estimate the value functions $\V^{r_i}(\hat{\pi}_i^t, \vpi_{-i}^{t-1})$ and $\V^{r_i}(\vpi^{t-1})$. For the former, the agents execute the policy $(\hat{\pi}_i^t, \vpi_{-i}^{t-1})$ for $M>0$ episodes\footnote{Each episode is a sequence of $H$ steps. At the beginning of each episode, the initial state is freshly sampled from $\mu$.} and agent $i$ estimates $\hat{\V}^{r_i}(\hat{\pi}_i^t, \vpi_{-i}^{t-1})$ with the average of the observed, cumulative rewards. For the latter, similarly, the agents execute $\vpi^{t-1}$ for $M$ episodes, but these observations can be used to estimate $\V^{r_1}(\vpi^{t-1}),...,\V^{r_n}(\vpi^{t-1})$ simultaneously\footnote{This holds because we assumed that the reward functions are known.}.
\end{enumerate}
\looseness -1 The resulting algorithm \AlgNameExp~(\AlgNameShortExp) is described in \cref{alg:cmpg-ca}.
\begin{algorithm}[t]
    \centering
\caption{\AlgNameShortExp~(Unknown Transitions)}\label{alg:cmpg-ca}
    \begin{algorithmic}[1]
\Require $\epsilon>0$ (approximation error), $\delta\in (0,1)$ (confidence), $\vpi^S\in\Pi_C$ (feasible policy), $T$ (number of iterations), $M>0$ (number of samples per policy)
\State $\vpi^0 \gets \vpi^S$
\For{$t=1,...,T$}
    \State Execute policy $\vpi^{t-1}$ for $M$ episodes and estimate $\hat{\V}^{r_1}(\vpi^{t-1}),...,\hat{\V}^{r_n}(\vpi^{t-1})$.
    \For{agent $i = 1,...,n$}
        \State Agent $i$ computes $\hat{\pi_i}^t$ such that \cref{eq:sample-efficient-cmdp-solver} is satisfied.
        \State Execute policy $(\hat{\pi}_i^t, \vpi_{-i}^{t-1})$ for $M$ episodes and estimate $\hat{\V}^{r_i}(\hat{\pi}_i^t, \vpi_{-i}^{t-1})$.
        \State $\epsilon_i^t \gets {\hat{\V}^{r_i}}({\hat{\pi}}_i^t, \vpi_{-i}^{t-1}) - {\hat{\V}^{r_i}}(\vpi^{t-1})$.
    \EndFor
    \If{$\max_{i\in[n]}\epsilon_i^t > \epsilon/2$}
        \State Set $\vpi^t = (\hat{\pi_j}^t, \vpi_{-j}^{t-1})$, where $j=\arg\max_{i\in[n]} \epsilon_i^t$, break ties arbitrarily.
    \Else
        \State \textbf{break}
    \EndIf
\EndFor
\end{algorithmic}
\end{algorithm}

\begin{restatable}{theorem}{thmcmpgcasc}
\label{thm:cmpg-ca-sample-complexity}
Given a strictly feasible CMPG $\G$ with Slater constant $\zeta>0$, suppose that the agents have access to an initial feasible policy (cf. Assumption \ref{assumption:cmpg-feasible-initial-policy}). Furthermore, assume that the agents have access to a sample-efficient CMDP solver (\cref{eq:sample-efficient-cmdp-solver}). Then, for any $\epsilon>0$, $\delta\in (0,1)$, \AlgNameShortExp~invoked with $M = \frac{32 H^2}{\epsilon^2} \log \paren{\frac{32 n^2 H}{\epsilon \delta}}$ and $T = \frac{4 nH}{\epsilon}$ returns an $\epsilon$-Nash policy with probability at least $1-\delta$, using the following number of samples:
\begin{align*}
    \F(\epsilon, \delta) \triangleq     \sum_{t=1}^T \sum_{i=1}^n \F_C\paren{|S|, |\Action_i|, H, \zeta, \frac{\epsilon \delta}{8 n^2 H}, \frac{\epsilon}{4}} + \frac{256 n^2 H^4}{\epsilon^3} \log \paren{\frac{32 n^2 H}{\epsilon \delta}}.
\end{align*}
\end{restatable}

\looseness -1 In the next two sub-sections, we will instantiate \AlgNameShortExp~with two different state-of-the-art CMDP solvers and state the resulting sample complexity bounds. Both algorithms are designed for CMDPs with a {\em single} constraint. Due to this, we set $k=1$ and denote our cost function by $\curly{ c_h }_{h=1}^H$ and refer to the constraint parameter as $\alpha$. Note that this is due to a limitation of the existing CMDP algorithms and not of \AlgNameShortExp.

\subsection{Generative model}
\label{subsec:cmdp-generative-model}
In this section, we assume that the agents have access to a {\em generative model}, i.e., they can directly obtain samples from the transition model $\Transition_h(\cdot|s,a)$, for any state-action pair $(s,a)\in\StateAction$ and any $h\in [H]$. 
Similar to previous results in CMDPs  \citet{vaswani2022nearoptimal} we propose a novel algorithm for finite-horizon CMDPs and describe it in \cref{alg:cmdp-generative-model}.
\begin{algorithm}[t]
    \centering
\caption{CMDPs with generative model}\label{alg:cmdp-generative-model}
    \begin{algorithmic}[1]
\Require $\StateSpace$ (state space), $\Action$ (action space), $H$ (horizon), $\curly{ r_h}_{h=1}^H$ (reward function), $\curly{ c_h}_{h=1}^H$ (constraint function), $\zeta>0$ (Slater constant), $N$ (number of samples), $\alpha'$ (constraint threshold), $U$ (projection upper
bound), $\lambda_0 = 0$ (initialization).
\State For each state-action $(s,a)$ pair and step $h\in[H]$, collect $N$ samples from $P_h(\cdot|s, a)$ and form the empirical transition model $\hat{\Transition}_h(\cdot|s,a)$. \label{algln:cmdp-gen-model-sampling}
\State Form the empirical CMDP $\hat{\M} = \paren{\StateSpace, \Action, \Horizon, \curly{\hat{\Transition}_h}_{h=1}^{\Horizon}, \mu, \curly{\reward_{h}}_{h=1}^{\Horizon},\curly{c_{h}}_{h=1}^{\Horizon}  , \alpha'}$.
\For{$t=0,...,T-1$} \label{algln:cmdp-gen-model-primal-dual}
\State Update the policy: $\hat{\pi}_t \in \arg\max_{\pi\in\Pi} \hat{\V}^{r -\lambda_t c}(\pi)$
\State Update the dual-variables: $\lambda_{t+1} =  \P_{[0,U]} \bracket{ \lambda_t - \eta \paren{ \alpha' - \hat{\V}^c(\hat{\pi}_t) }} $
\EndFor \label{algln:cmdp-gen-model-primal-dual-end}
\State Convert $\curly{\hat{\pi}_t}_{t=0}^{T-1}$ into a single policy $\bar{\pi}$ s.t. $\hat{\V}^l(\bar{\pi}) = \frac{1}{T} \sum_{t=1}^T \hat{\V}^l(\hat{\pi}_t)$ for $l=r,c$ (see \cref{subsec:occupancy-measures}). \label{algln:cmdp-gen-model-convert}
\end{algorithmic}
\end{algorithm}
\cref{lem:cmdp-sample-complexity-gen-model} (cf.~\cref{app:cmdp-generative-model}) establishes the sample complexity for \cref{alg:cmdp-generative-model}.

\begin{corollary}
\label{cor:cmpg-ca-sc-gen-model}
\looseness -1 Given a strictly feasible CMPG $\G$, assume that its Slater constant $\zeta>0$ is known. Furthermore, assume that the agents invoke \cref{alg:cmdp-generative-model} with $\epsilon'=\frac{\epsilon}{4}$, $\delta'=\O\paren{\frac{\epsilon \delta}{n^2 H}}$ and parameters set as in \cref{lem:cmdp-sample-complexity-gen-model} to solve \cref{eq:sample-efficient-cmdp-solver}. Then, for any $\epsilon>0, \delta \in (0,1)$, \AlgNameShortExp~invoked with $M=\O\paren{\frac{H^2}{\epsilon^2} \log\paren{\frac{n H}{\epsilon \delta}}}$ and $T=\frac{4 n H}{\epsilon}$, returns an $\epsilon$-Nash policy with probability at least $1-\delta$ with an overall sample complexity of:
\begin{align*}
\F(\epsilon,\delta) &\leq \wtilde{\O}\paren{ \frac{n |\StateSpace| H^8 \log \paren{\frac{1}{\epsilon \delta}} \sum_{i=1}^n|\Action_i| }{\epsilon^3 \zeta^2} + \frac{n^2 H^4 \log \paren{\frac{1}{\epsilon \delta}} }{\epsilon^3}}.
\end{align*}
\end{corollary}

\textbf{Remark:} Compared to the result for {\em unconstrained} MPGs \citep[Theorem 7]{Song2021}, our \cref{cor:cmpg-ca-sc-gen-model} has an additional dependence on $\frac{1}{\zeta^2}$ and a worse dependence on the horizon $H$. These are due to the fact that our CMDP solver must always return a {\em feasible} policy. Finally, the sample complexity result in \citet{Song2021} explicitly depends on $\Phi_{max} \triangleq \max_{\vpi\in\Pi}\mPhi(\vpi)$, whereas we substituted $\Phi_{\max}\leq nH$ (\cref{lem:potential-bounded}).

\subsection{Safe exploration without a generative model}
\label{subsec:cmdp-no-regret}
We now consider the more challenging setting where the agents do not have access to a generative model, but can only explore by executing policies and observing the transitions.
Moreover, during the learning process, we want to ensure that the agents explore {\em safely}. 
Existing algorithms with safe exploration \citep{Bura2022, liu2021learning} have guarantees on the {\em regret}, but no sample complexity guarantees. To address this, we derive a sample complexity bound for the algorithm by \citet{Bura2022} (\cref{alg:cmdp-no-regret}) in \cref{lem:cmdp-sample-complexity-no-regret}.

\begin{algorithm}[t]
    \centering
\caption{CMDPs with safe exploration}\label{alg:cmdp-no-regret}
\begin{algorithmic}[1]
\Require $T$ (total number of iterations), $\delta\in (0,1)$ (confidence), $(\pi^S,\alpha_S)$ (strictly feasible policy and its constraint value), $M>0$ (episodes per policy)
\State Execute $\Dope(\delta,\pi^S, \alpha_S, \alpha, T)$ and obtain policies $\pi_1,...,\pi_T$. \label{algln:dopa}
\For{$t=1,...,T$}
\State Execute policy $\pi_t$ for $M$ episodes and form the estimate $\hat{\V}^r(\pi_t)$.
\EndFor
\State Return policy $\hat{\pi} \in \arg\max_{\pi \in \curly{\pi_t}_{t\in [T]}} \hat{\V}^r(\pi)$. \label{algln:no-regret-final-policy}
\end{algorithmic}
\end{algorithm}

To apply this CMDP solver in \AlgNameShortExp, we need to ensure that in every iteration, the agents have access to a {\em strictly} feasible policy. We state a stronger condition in the following assumption.
\begin{assumption}
    \label{assumption:cmdp-no-regret}
    There exists $c\in (0,\zeta]$ s.t. for any agent $i\in [n]$ and policy $\vpi_{-i}\in\Pi_C\setminus{\Pi^i}$ of the other agents, the agent can obtain a strictly feasible policy $\pi\in\Pi^i$ s.t. $\V^c(\pi, \vpi_{-i}) \leq \alpha - c$.
\end{assumption}
This is a stronger assumption than in \cref{subsec:cmdp-generative-model}, as we additionally require access to a strictly feasible policy for every CMDP that is solved in \AlgNameShortExp.
\begin{corollary}
\label{cor:cmpg-ca-sc-no-regret}
\looseness -1 Suppose that Assumption \ref{assumption:cmdp-no-regret} holds. Given $\epsilon>0, \delta \in (0,1)$, assume that we invoke \AlgNameShortExp~with $M=\O\paren{\frac{H^2}{\epsilon^2} \log\paren{\frac{n H}{\epsilon \delta}}}$ and $T=\frac{4 n H}{\epsilon}$. Furthermore, assume that we use \cref{alg:cmdp-no-regret} as CMDP solver with $\epsilon'=\frac{\epsilon}{4}$, $\delta'=\O\paren{\frac{\epsilon \delta}{n^2 H}}$ and parameters set as in \cref{lem:cmdp-sample-complexity-no-regret}. Then, \AlgNameShortExp~returns an $\epsilon$-Nash policy with probability at least $1-\delta$ with an overall sample complexity of:
\begin{align*}
\F(\epsilon,\delta) &\leq \wtilde{\O}\paren{ \frac{n |\StateSpace|^2 H^{10} \log \paren{\frac{1}{\epsilon \delta}} \sum_{i=1}^n|\Action_i| }{\epsilon^5 c^2 } + \frac{n^2 H^4 \log \paren{\frac{1}{\epsilon \delta}} }{\epsilon^3}}.
\end{align*}
\end{corollary}
Note that to satisfy Assumption \ref{assumption:cmdp-no-regret}, any $c\in (0,\zeta]$ is a valid choice. A large $c$ yields a better sample complexity for \cref{cor:cmpg-ca-sc-no-regret}, but restricts the set of strictly feasible policies for the CMDP solver. A smaller $c$ increases the sample complexity, but gives more flexibility, as it allows for a larger set of strictly feasible policies. Comparing the two corollaries, we observe that safe exploration without a generative model leads to a worse dependence of $|S|$, $\Horizon$ and $\epsilon$.

\definecolor{brightBlue}{RGB}{68, 119, 170}
\definecolor{brightRed}{RGB}{238, 102, 119}

\vspace{-1mm}
\section{Experiments}
\vspace{-1mm}
\label{sec:experiments}
\subsection{Grid world}
\looseness -1 We consider a cooperative CMPG with two agents, in which the agents navigate in a 4x4 grid world (cf. \cref{fig:grid-world-setup}). Each cell in the grid represents a state and in every state, each agent can choose to move {\em up}, {\em right}, {\em down} or {\em left}. State transitions are deterministic and if an agent selects an action that would make it leave the grid, it remains in the current state. \cref{fig:grid-world-setup} illustrates the rewards that an agent can obtain in the individual states. Both agents start from the bottom left state and their goal is to reach the target state, which is the state with a reward of 10. To model this as a cooperative game, we set the agents' joint objective to be the sum of their individual rewards. Whenever the agents are on the same state, excluding the start and target states, they {\em collide} and incur a cost of 1. The agents must keep the expected cost below a pre-defined threshold $\alpha\in [0,1]$.

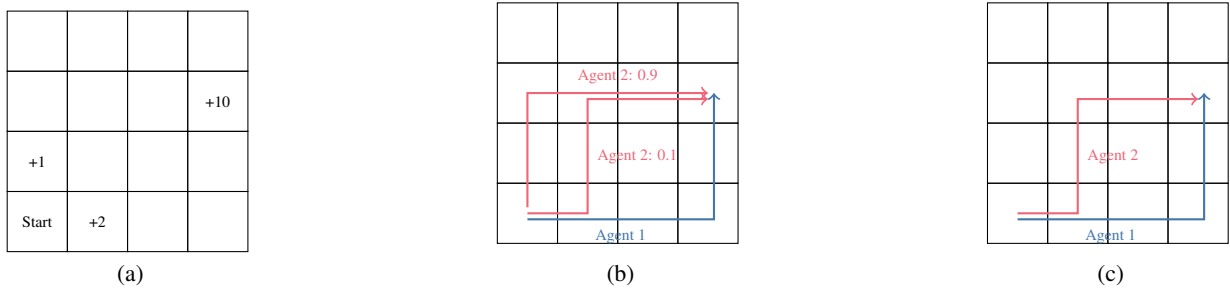
\begin{figure}
    \centering
\subfloat[]{
	\begin{tikzpicture}[scale=0.8]
	\foreach \x in {0,1,2,3}
	\foreach \y in {0,1,2,3}
	{
		\draw (\x,\y) rectangle (\x+1,\y+1);
	}
        \node[font=\tiny] at (0.5,0.5) {Start};
        \node[font=\tiny] at (1.5,0.5) {+2};
        \node[font=\tiny] at (0.5,1.5) {+1};
        \node[font=\tiny] at (3.5,2.5) {+10};
	\end{tikzpicture} 
        \label{fig:grid-world-setup}
}
\hfill
\subfloat[]{
        \begin{tikzpicture}[scale = 0.8]
	\foreach \x in {0,1,2,3}
	\foreach \y in {0,1,2,3}
	{
		\draw (\x,\y) rectangle (\x+1,\y+1);
	}
        \draw[brightBlue, thick, solid, ->] (0.5,0.4) -- node[below,font=\tiny]{Agent 1}(3.6,0.4) -- (3.6,2.5);
        \draw[brightRed, thick, solid, ->] (0.5,0.5) -- (1.5,0.5) -- node[right,font=\tiny]{Agent 2: 0.1}(1.5,2.4) -- (3.5,2.4);
        \draw[brightRed, thick, solid, ->] (0.5,0.6) -- (0.5,2.5) -- node[above,font=\tiny]{Agent 2: 0.9}(3.5,2.5);
	\end{tikzpicture}
        \label{fig:grid-world-policies}
}
\hfill
\subfloat[]{
        \begin{tikzpicture}[scale = 0.8]
	\foreach \x in {0,1,2,3}
	\foreach \y in {0,1,2,3}
	{
		\draw (\x,\y) rectangle (\x+1,\y+1);
	}
        \draw[brightBlue, thick, solid, ->] (0.5,0.4) -- node[below,font=\tiny]{Agent 1}(3.6,0.4) -- (3.6,2.5);
        \draw[brightRed, thick, solid, ->] (0.5,0.5) -- (1.5,0.5) -- node[right,font=\tiny]{Agent 2}(1.5,2.4) -- (3.5,2.4);
	\end{tikzpicture}
        \label{fig:grid-world-policies-dual}
}
    \caption{Grid world experiment: \cref{fig:grid-world-setup} illustrates the state space that the agents navigate in. Both agents start from the bottom left state and their goal is to maximize the sum of their individual rewards. The numbers on the states indicate the rewards associated with those states. The choice of parameters for our evaluation is described in \cref{sec:experiments}. \cref{fig:grid-world-policies} displays the policies with their corresponding probabilities returned by \AlgNameShort. If the agents were to solve the dual problem directly (\cref{sec:cmpg-duality}), they might obtain the policy illustrated in \cref{fig:grid-world-policies-dual}, which is not feasible.}
    \label{fig:grid-world}
\end{figure}

\looseness -1 We evaluate our algorithm \AlgNameShort~ with known transitions and use a primal-dual algorithm as CMDP solver. We set the horizon to $H=6$ and use a threshold of $\alpha=0.1$.
\cref{fig:results-grid} (top row) displays the reward differences between the current policy and the new policy for both agents and after every cycle of the algorithm, averaged over 20 runs. One cycle corresponds to one full iteration of \cref{alg:cmpg-ca}, i.e. all agents solving their CMDPs. When the reward differences reach zero for both agents, this implies that the agents have converged to a Nash policy. The bottom row tracks the cost over the cycles of \cref{alg:cmpg-ca}. The agents start from a strictly feasible policy with a cost of 0, and converge to a policy with a cost close to $\alpha$.
\begin{figure}
    \centering
  \subfloat[Grid world]{\includegraphics[width=0.45\textwidth]{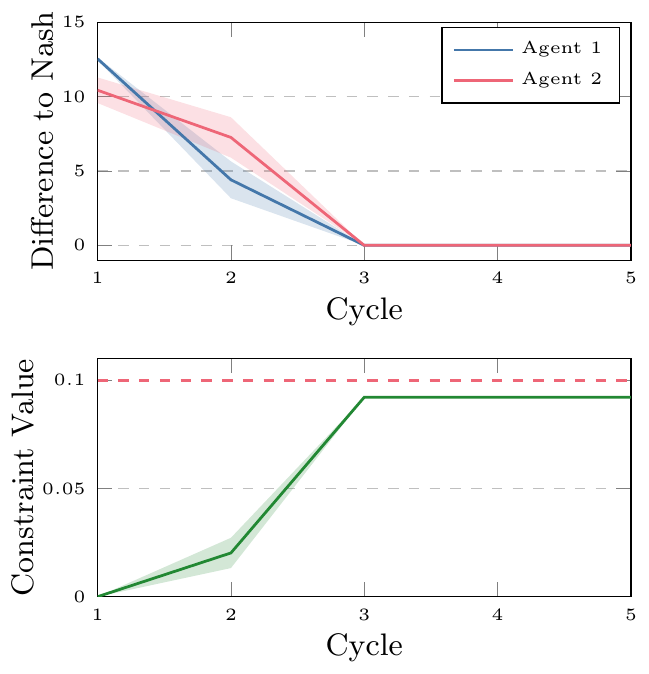}\label{fig:results-grid}}
  \hfill
  \subfloat[Congestion game]{\includegraphics[width=0.45\textwidth]{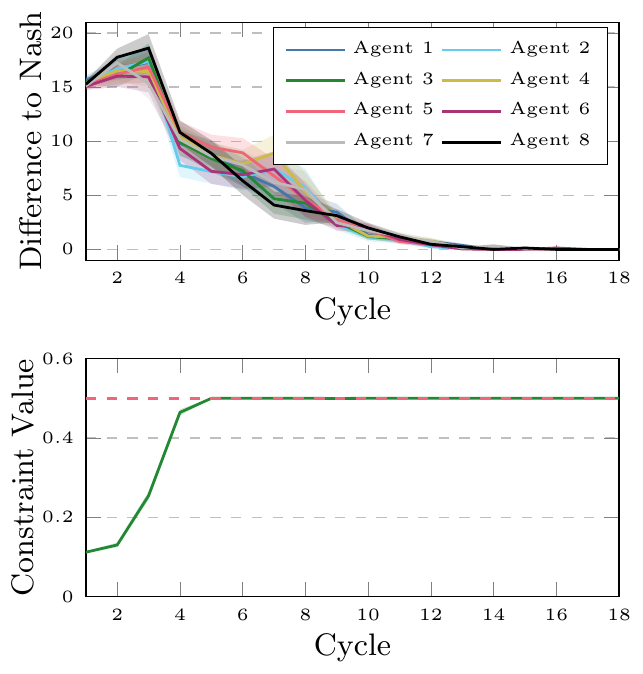}\label{fig:results-congestion}}
    \caption{These plots illustrate the results of the grid world (\cref{fig:results-grid}) and congestion game (\cref{fig:results-congestion}) experiments. One cycle on the x-axis corresponds to one full iteration of \cref{alg:cmpg-ca}, i.e. all $N$ agents solving their CMDPs. The top row displays, for each agent, an average of their reward difference between the current and new policy. When the difference reaches zero, they have converged to a Nash policy. The bottom row tracks the averaged cost over the cycles of \cref{alg:cmpg-ca} (green, solid line). The agents start from a strictly feasible policy and converge to a cost close to $\alpha$ (red, dashed line). In all plots, we additionally also display the standard error.}
    \label{fig:experiments-results}
\end{figure}

The resulting policies with the corresponding probabilities are shown in \cref{fig:grid-world-policies}. With this, the agents collide once with probability 0.1, thus, satisfying the constraint of the experiment. On the other hand, if we solve the Lagrangian dual problem directly (\cref{sec:cmpg-duality}), one of the returned policies is illustrated in \cref{fig:grid-world-policies-dual}. In this case, they always have one collision, which does not satisfy the constraint of the experiment.

\looseness -1
\subsection{Congestion game} We consider a finite-horizon version of the setup described in \citet{leonardos2022global}, i.e., a non-cooperative MPG in which every state is a congestion game\footnote{Every congestion game is also a potential game and vice versa \citep{monderer1996potential}.}. The game consists of two states $\StateSpace = \curly{\mathtt{safe}, \mathtt{unsafe}}$, $N$ agents and action space $\Action = \curly{A,B,C,D}$ for every agent. Each action $a\in\Action$ in state $s\in \StateSpace$ has a weight $w_a^s>0$ associated with it. In the safe state, an agent that selects action $a\in\Action$, receives a reward of $k_a\cdot w_a^{\mathtt{safe}}$, where $k_a$ denotes the number of agents that selected action $a$. In the unsafe state, the reward structure is similar, however, we subtract an offset $c\geq 0$, resulting in a reward of $k_a\cdot w_a^{\mathtt{unsafe}} - c$. In both states $s\in\StateSpace$, the weights follow the order $w_A^s<w_B^s<w_C^s<w_D^s$. Thus, in both states, the agents prefer to take the action that is chosen by most agents. Furthermore, for every action $a\in\Action$, $k_a \cdot w_a^{\mathtt{safe}} \gg k_a \cdot w_a^{\mathtt{unsafe}} - c$ s.t. the agents prefer to stay in the safe state. In the safe state, if more than $N/2$ agents choose the same action, the system transitions to the unsafe state. To get back to the safe state from the unsafe state, the agents must equally distribute themselves among the four actions. The transitions are illustrated in \cref{fig:congestion-game-transitions}.

\begin{figure}[h]
\centering
\begin{tikzpicture}[->,>=stealth',shorten >=1pt,auto,node distance=3cm,semithick,scale=0.8]
  \node[state,font=\tiny] (s1) {$\mathtt{safe}$};
  \node[state,font=\tiny] (s2) [right of=s1] {$\mathtt{unsafe}$};

  \path (s1) edge [bend left,font=\tiny] node {$k^* > N/2$} (s2)
            edge [loop above,font=\tiny] node {$k^* \leq N/2$} (s1)
        (s2) edge [bend left,font=\tiny] node {$k^* \leq N/4$} (s1)
            edge [loop above,font=\tiny] node {$k^* > N/4$} (s2);
\end{tikzpicture}
\caption{Congestion game experiment: For every action $a\in \Action$, we denote by $k_a$ the number of agents that select $a$ in the current step. This figure visualizes the state transitions in every step, where $k^* \triangleq \max_{a\in \Action} k_a$ denotes the maximum number of agents that have selected the same action.}
\label{fig:congestion-game-transitions}
\end{figure}
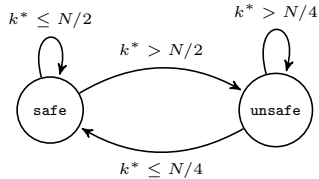

\looseness -1 We evaluate our algorithm \AlgNameShort~with $N=8$ agents and a horizon of $H=2$. Furthermore, we assume that the transitions are known and use a linear program to solve the CMDPs \citep{altman1999constrained}. For the initial state, we set $\mu(\mathtt{safe}) = \mu(\mathtt{unsafe}) = 0.5$. At step $h=1$, in the unsafe state, if more than $N/2$ agents select the same action, the agents incur a cost of 1. Their goal is to keep the cost below a threshold $\alpha=0.5$. 
\cref{fig:results-congestion} (top row) displays, as before, the reward differences between the current policy and the new policy, for each agents and averaged over 50 runs. When this difference reaches zero, this implies that the agents have converged to a Nash policy. The bottom plots track the cost over the cycles of \cref{alg:cmpg-ca}. The agents start from a strictly feasible policy with a cost of 0, and converge to a value close to $\alpha$.

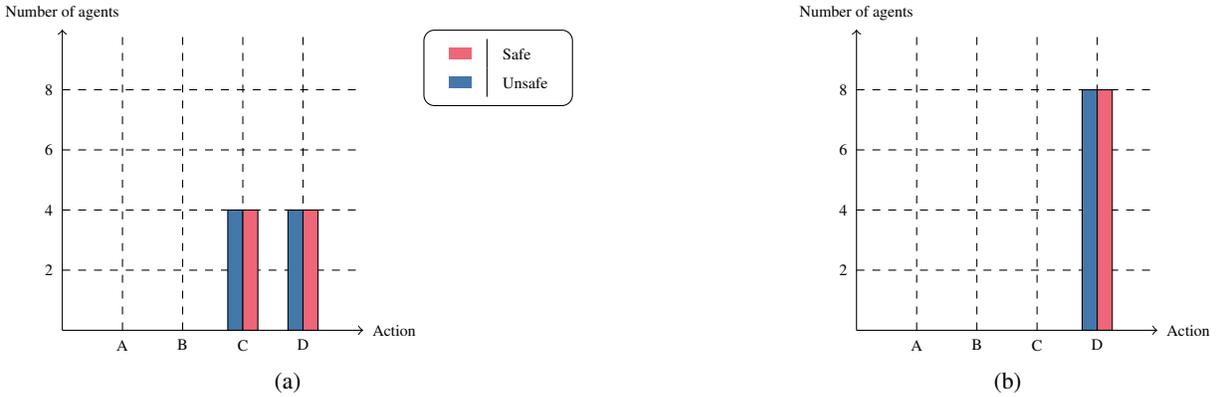
\begin{figure}[h]
    \centering

\subfloat[]{
\begin{tikzpicture}[scale=0.8]
  \draw[->] (0,0) -- (5,0) node[right,font=\tiny] {Action};
  \draw[->] (0,0) -- (0,5) node[above,font=\tiny] {Number of agents};
  \draw[dashed] (1,0) -- (1,5);
  \draw[dashed] (2,0) -- (2,5);
  \draw[dashed] (3,0) -- (3,5);
  \draw[dashed] (4,0) -- (4,5);
  \draw[dashed] (0,1) -- (5,1);
  \draw[dashed] (0,2) -- (5,2);
  \draw[dashed] (0,3) -- (5,3);
  \draw[dashed] (0,4) -- (5,4);

  \draw[fill=brightBlue] (2.75,0) rectangle (3,2);
  \draw[fill=brightBlue] (3.75,0) rectangle (4,2);
  \draw[fill=brightRed] (3,0) rectangle (3.25,2);
  \draw[fill=brightRed] (4,0) rectangle (4.25,2);

  \node[below,font=\tiny] at (1,0) {A};
  \node[below,font=\tiny] at (2,0) {B};
  \node[below,font=\tiny] at (3,0) {C};
  \node[below,font=\tiny] at (4,0) {D};

  \node[left,font=\tiny] at (0,1) {2};
  \node[left,font=\tiny] at (0,2) {4};
  \node[left,font=\tiny] at (0,3) {6};
  \node[left,font=\tiny] at (0,4) {8};

    \node[draw,rounded corners,anchor=north west] at (6,5) {
    \begin{tabular}{c|l}
      \tiny{\textcolor{brightRed}{\rule{0.3cm}{0.15cm}}} & \tiny{Safe} \\
      \tiny{\textcolor{brightBlue}{\rule{0.3cm}{0.15cm}}} & \tiny{Unsafe} \\
    \end{tabular}
  };
\end{tikzpicture}
\label{fig:congestion-game-1}
}
\hfill
\subfloat[]{
\begin{tikzpicture}[scale=0.8]
  \draw[->] (0,0) -- (5,0) node[right,font=\tiny] {Action};
  \draw[->] (0,0) -- (0,5) node[above,font=\tiny] {Number of agents};
  \draw[dashed] (1,0) -- (1,5);
  \draw[dashed] (2,0) -- (2,5);
  \draw[dashed] (3,0) -- (3,5);
  \draw[dashed] (4,0) -- (4,5);
  \draw[dashed] (0,1) -- (5,1);
  \draw[dashed] (0,2) -- (5,2);
  \draw[dashed] (0,3) -- (5,3);
  \draw[dashed] (0,4) -- (5,4);

  \draw[fill=brightBlue] (3.75,0) rectangle (4,4);
  \draw[fill=brightRed] (4,0) rectangle (4.25,4);

  \node[below,font=\tiny] at (1,0) {A};
  \node[below,font=\tiny] at (2,0) {B};
  \node[below,font=\tiny] at (3,0) {C};
  \node[below,font=\tiny] at (4,0) {D};

  \node[left,font=\tiny] at (0,1) {2};
  \node[left,font=\tiny] at (0,2) {4};
  \node[left,font=\tiny] at (0,3) {6};
  \node[left,font=\tiny] at (0,4) {8};
\end{tikzpicture}
\label{fig:congestion-game-2}
}
    \caption{Congestion game experiment: The choice of parameters and constraint function is detailed in \cref{sec:experiments}. \cref{fig:congestion-game-1} and \cref{fig:congestion-game-2} plot the resulting distributions over the actions for steps $h=1$ and $h=2$, respectively.}
    \label{fig:congestion-game-distributions}
\end{figure}

\cref{fig:congestion-game-1} and \cref{fig:congestion-game-2} plot the resulting distributions over the actions for steps $h=1$ and $h=2$, respectively. We observe that in step $h=1$, if the agents start from the safe state, they select their actions s.t. in step $h=2$, the system remains in the safe state. At step $h=2$, the agents maximize their rewards by selecting action D, irrespective of the state that the system is in. At step $h=1$, without the constraints, all agents would prefer to choose action D at the unsafe state. With the choice of our constraints, as we can observe in \cref{fig:congestion-game-1}, the agents distribute themselves equally amongst actions C and D.
\vspace{-1mm}
\section{Conclusion}
\vspace{-1mm}
In this paper, we proved that strong duality does not hold always hold in CMPGs, making primal-dual approaches inapplicable. An interesting future question could be to understand under which conditions primal-dual methods may work for CMPGs. To tackle CMPGs, we presented our algorithm \AlgNameShort, which provably converges to an $\epsilon$-Nash policy.  Note that while this paper focuses on the finite-horizon setting, our algorithm \AlgNameShort~can be adapted to the discounted, infinite-horizon setting by using an appropriate CMDP solver as a sub-routine. Furthermore, we established the first sample complexity bound for learning in CMPGs. In \AlgNameShort, exploration happens only within the CMDP sub-routines. It would be interesting to understand whether the sample complexity bound for the generative model setting (\cref{subsec:cmdp-generative-model}) can be made tighter if we move the exploration outside the CMDP sub-routines.

\begin{acksection}
We thank Daniil Dmitriev, Manish Prajapat and Vignesh Ram Somnath for their valuable comments on the paper. This research was primarily supported by the ETH AI Center. Pragnya Alatur has been funded in part by ETH Foundations of Data Science (ETH-FDS). Giorgia Ramponi is partially funded by Google Brain. 
\end{acksection}

\bibliography{references.bib}
\newpage
\appendix
\doparttoc
\faketableofcontents
\part{Supplementary Material}
\parttoc
\section{Duality for Constrained Markov Potential Games? (\cref{sec:cmpg-duality})}
\label{app:cmpg-duality}
This section provides the theoretical proofs for \cref{sec:cmpg-duality}. We start by proving that a Nash policy can be found via constrained optimization in the following proposition.
\lemconstrglobalmaxnash*
\begin{proof}
We prove this by contradiction. Suppose that $\vpi^*$ is not a Nash policy. Therefore, there is an agent $i\in [n]$, for which there exists a policy $\hat{\pi}_i \in \arg\max_{\pi\in \Pi^i_C\paren{\vpi^*_{-i}}} \V^{r_i}\paren{\pi, \vpi_{-i}^*}$ such that $ \V^{r_i}\paren{\hat{\pi}_i, \vpi_{-i}^*} > \V^{r_i}\paren{\vpi^*}$. Thus, agent $i$ can {\em strictly increase} its value function by deviating to $\hat{\pi}_i$. This implies that we can also increase the potential function, i.e.
\begin{align*}
    \mPhi(\hat{\pi}_i, \vpi_{-i}^*) - \mPhi(\vpi^*) &= \V^{r_i}(\hat{\pi}_i, \vpi_{-i}^*) - \V^{r_i}(\vpi^*) > 0 
    \quad \text{(By \cref{eq:potential-property}.)}\\
    \Rightarrow \mPhi(\hat{\pi}_i, \vpi_{-i}^*) > \mPhi(\vpi^*).
\end{align*}
\looseness -1 Since $(\hat{\pi}_i, \vpi^*_{-i}) \in \Pi_C$ and $\mPhi(\hat{\pi}_i, \vpi_{-i}^*) > \mPhi(\vpi^*)$, this contradicts our assumption that $\vpi^*$ is a solution to \cref{eq:cmpg-optimization-problem}. Therefore, $\vpi^*$ must be a Nash policy.
\end{proof}

As stated in \cref{sec:cmpg-duality}, we make use of the constrained optimization formulation in \cref{lem:constr-global-max-nash} to define the Lagrangian primal and dual problems in \cref{eq:primal} and \cref{eq:dual}, respectively. We are interested in solving the dual problem, as it corresponds to a modified, {\em unconstrained} MPG, which we prove in the following lemma. 

\lemdualmodifiedmpg*
\begin{proof}
Consider an arbitrary $\vlambda\in\Reals_+^k$. Then, we can write the new value function, for any agent $i\in[n]$, as follows:
\begin{equation*}
    \V^{\tilde{r}_i}(\vpi) = \E\bracket{\sum_{h=1}^{\Horizon} \tilde{r}_{i,h}(s_h,a_h) | s_0 = s},
\end{equation*}
where the expectation is taken with respect to $\mu, \vpi$ and $\Transition$. Next, we plug in the definition for $\wtilde{r}_i$ and obtain:
\begin{align}
\V^{\tilde{r}_i}&=\E\bracket{\sum_{h=1}^{\Horizon} \paren{ r_{i,h}(s_h,a_h) + \sum_{j=1}^k \lambda_j c_{j,h}(s_h,a_h) } \Big| s_0 = s} \nonumber\\
    &= \E\bracket{\sum_{h=1}^{\Horizon} r_{i,h}(s_h,a_h) \Big| s_0 = s} + \sum_{j=1}^k \lambda_j \E\bracket{\sum_{h=1}^{\Horizon}  c_{j,h}(s_h,a_h) \Big| s_0 = s} \label{proof:dual-modified-mpg-loe}\\
    &= \V^{r_i}(\vpi) + \sum_{j=1}^k \lambda_j \V^{c_j}(\vpi), \nonumber
\end{align}
where \cref{proof:dual-modified-mpg-loe} is due to linearity of expectation. Next, we show that $\L(\cdot, \vlambda)$ is indeed a potential function for the value functions $\curly{\V^{\widetilde{r}_i
}}_{i\in[n]}$. For any $\vpi\in\Pi$ and $\pi_i\in\Pi^i$, we evaluate the difference in $\L(\cdot,\vlambda)$ between $\vpi$ and $(\pi_i,\vpi_{-i})$:
\begin{align*}
    \L\paren{ \paren{\pi_i, \vpi_{-i}}, \vlambda } - \L\paren{ \vpi, \vlambda } &= \paren{\mPhi(\pi_i, \vpi_{-i}) + \sum_{j=1}^k \lambda_j \paren{ \V^{c_j}\paren{\pi_i, \vpi_{-i}} - \alpha_j }} - \\ 
    &\quad\quad \paren{ \mPhi(\vpi) + \sum_{j=1}^k \lambda_j  \paren{ \V^{c_j}(\vpi) - \alpha_j } }  \tag{By \cref{eq:lagrangian}.}\\
    &=  \mPhi(\pi_i, \vpi_{-i}) - \mPhi(\vpi) + \sum_{j=1}^k \lambda_j \paren{ \V^{c_j}(\pi_i, \vpi_{-i}) - \V^{c_j}(\vpi) } \\
    &= \V^{r_i}(\pi_i, \vpi_{-i}) - \V^{r_i}(\vpi) + \tag{By \cref{eq:potential-property}.}\\
    &\quad\quad \sum_{j=1}^k \lambda_j \paren{ \V^{c_j}(\pi_i, \vpi_{-i}) - \V^{c_j}(\vpi) } \\
    &= \V^{\tilde{r}_i}(\pi_i, \vpi_{-i}) - \V^{\tilde{r}_i}(\vpi).
\end{align*}
The last equality implies that $\L(\cdot,\vlambda)$ satisfies \cref{eq:potential-property} for the new value functions. We conclude that $\L(\cdot,\vlambda)$ is indeed potential function for the modified rewards.
\end{proof}

While \cref{lem:dual-modified-mpg} proves that the dual problem reduces to an unconstrained MPG, which can be solved using existing techniques \citep{leonardos2022global}, we prove in  \cref{thm:strong-duality-cmpgs} (cf. \cref{sec:cmpg-duality}), that, unfortunately, strong duality does not always hold for CMPGs. Moreover, we show that the resulting policy is not only sub-optimal but also  does not respect the constraints.

What if a CMPG satisfies the strong duality property?  Next, we modify the CMPG in \cref{eq:counter-example} such that strong duality holds. We will demonstrate that even in that case, the dual problem might not return a feasible Nash policy. Consider a modified version of the CMPG in  \cref{eq:counter-example} with $A$ defined as follows:
\begin{align} \label{eq:counter-example-modified}
    A = \begin{bmatrix}
3 & 3\\
3 & 4 \end{bmatrix}.
\end{align}

\textbf{Primal problem:}
First, we recall the primal problem, which is defined as follows:
\begin{equation*}
    P^* = \max_{\pi_1,\pi_2\in\Delta_2} \min_{\lambda\in\Reals_+} \paren{\pi_1^T A \pi_2  + \lambda \paren{\frac{1}{2} - \pi_1^T B \pi_2}}
\end{equation*}
It is easy to see that $P^*=3.5$ here. The following policies return an expected reward of $P^*$:
\begin{enumerate}
    \item $\pi_1 = \bracket{\frac{1}{2}, \frac{1}{2}}, \pi_2 = \bracket{0, 1}$
    \item $\pi_1 = \bracket{0,1}, \pi_2 = \bracket{\frac{1}{2}, \frac{1}{2}}$
\end{enumerate}

\textbf{Dual problem:}
Next, we recall the dual problem, which is defined as follows:
\begin{equation*}
    D^* =  \min_{\lambda\in\Reals_+} \underbrace{\max_{\pi_1,\pi_2\in\Delta_2} \paren{\pi_1^T A \pi_2  + \lambda \paren{ \frac{1}{2} - \pi_1^T B \pi_2}}}_{=:d(\lambda)},
\end{equation*}
\begin{figure}
    \centering
    \includegraphics{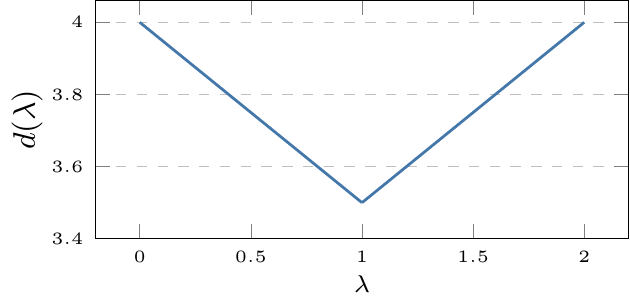}
    \caption{This figure displays the dual function $d(\lambda)$ for the CMPG in \cref{eq:counter-example-modified},  evaluated at 1000 equidistant locations $\lambda\in[0,2]$.}
    \label{fig:dual-duality-holds}
\end{figure}
\cref{fig:dual-duality-holds} visualizes the dual function in the interval $\lambda\in [0,2]$. Using the same reasoning as before, it is easy to see that the optimal dual variable is $\lambda_D^* = 1$, resulting in the dual value $D^* = 3.5$. Note that in this CMPG, strong duality holds, i.e., $P^* = D^*$. Next, we list three policies that are solutions to the dual problem, i.e., policies that satisfy $\pi_D^*\in \arg\max_{\pi} \curly{\pi_1^T A \pi_2 + \lambda_D^* \paren{\frac{1}{2} - \pi_1^T B \pi_2}}$:
\begin{enumerate}
    \item $\pi_1 = \bracket{1, 0}, \pi_2 = \bracket{1, 0}$
    \item $\pi_1 = \bracket{0, 1}, \pi_2 = \bracket{0, 1}$
    \item $\pi_1 = \bracket{\frac{1}{2}, \frac{1}{2}}, \pi_2 = \bracket{\frac{1}{2}, \frac{1}{2}}$
\end{enumerate}
The first policy is feasible and is indeed a Nash policy. As established before, the second policy does not satisfy the constraints. The third policy is feasible, {\em but it is not a Nash policy}. For example, agent 1 can improve the expected reward by switching to policy $\pi_1 = \bracket{0,1}$, and vice versa. Therefore, even if strong duality hold, the dual problem may not return a Nash policy.

\section{Solving Constrained Markov Potential Games (\cref{sec:solving-cmpgs})}
\label{app:solving-cmpgs}

This section provides the theoretical proofs for \cref{sec:solving-cmpgs}.
To prove convergence of \AlgNameShort, we require the potential function to be bounded, which we prove in the following lemma.
\lempotentialbounded*
\begin{proof}
    For any $\vpi\in\Pi$, we define a new sequence of policies $\curly{\wtilde{\vpi}^k}_{k=0,...,n}$ as follows:
    \begin{equation}
        \wtilde{\vpi}^k \triangleq 
        \begin{bmatrix}
        \pi^B_1 \\
        \vdots \\
        \pi^B_{k} \\
        \pi_{k+1} \\
        \vdots \\
        \pi_{n}
    \end{bmatrix}, \;\;\;\; k=0,...,n.
    \end{equation}
    It is easy to see that $\wtilde{\vpi}^0 \equiv \vpi$ and $\wtilde{\vpi}^n \equiv \vpi^B$. For any agent $i\in [n]$, the policies $\wtilde{\vpi}^{i-1}$ and $\wtilde{\vpi}^i$ differ only in agent i. By applying the potential property (\cref{eq:potential-property}), we obtain:
    \begin{equation*}
        \mPhi(\wtilde{\vpi}^{i-1}) - \mPhi(\wtilde{\vpi}^i) = \V^{r_i}(\wtilde{\vpi}^{i-1}) - \V^{r_i}(\wtilde{\vpi}^i).
    \end{equation*}
    Summing over all $i\in [n]$ and rearranging the terms, we get:
    \begin{align*}
        \mPhi(\vpi) &= \sum_{i=1}^n \bracket{ \V^{r_i}\paren{\wtilde{\vpi}^{i-1}} - \V^{r_i}\paren{\wtilde{\vpi}^i} } + \mPhi(\vpi^B) \\
        &\leq \sum_{i=1}^n \underbrace{\abs{ \V^{r_i}\paren{\wtilde{\vpi}^{i-1}} - \V^{r_i}\paren{\wtilde{\vpi}^i}}}_{\leq H} + \mPhi(\vpi^B) \\
        &\leq nH + \mPhi(\vpi^B) \tag{Since $r_{i,h} \in [0,1], \forall i\in [n], \forall h\in [H]$).}
    \end{align*}
\end{proof}

We are now ready to prove the convergence of \cref{alg:cmpg-ca-no-exploration}. For completeness, we repeat the statement of \cref{thm:convergence-cmpg-ca-no-exploration} here.

\thmconvergencecmpgnoexp*
\begin{proof}
Consider time step $t$ and assume that the algorithm has not converged yet. This implies that there is an agent $i\in [n]$ s.t. $\V^{r_i}(\hat{\pi}_i^t, \vpi^{t-1}_{-i}) - \V^{r_i}(\vpi^{t-1})>\epsilon/2$; therefore, $\mPhi(\vpi^t)-\mPhi(\vpi^{t-1})>\epsilon/2$ holds due to the potential property (\cref{eq:potential-property}). 

By applying \cref{lem:bound-dual-variable} with the initial policy $\vpi^0$, we know that $\mPhi(\vpi)\leq nH + \mPhi(\vpi^0)$ holds, for all $\vpi\in\Pi$. Since, in every iteration, the potential function increases by at least $\epsilon/2$, we must have $\V^{r_i}(\hat{\pi}_i^{T}, \vpi_{-i}^{T-1}) -  \V^{r_i}\paren{\vpi^{T-1}}\leq \epsilon/2$, $\forall i\in [n]$ after $T$ iterations.

Combining this with \cref{eq:cmpg-ca-cmdp}, we obtain the following guarantee for every agent $i\in [n]$:
\begin{equation*}
    \max_{\pi \in \Pi^i_C(\vpi^{T-1}_{-i})} \V^{r_i}(\pi, \vpi_{-i}^{T-1}) - \V^{r_i}(\hat{\pi}_i^T, \vpi_{-i}^{T-1}) \leq \epsilon.
\end{equation*}
Finally, recall that we start from a feasible policy $\vpi^0\in\Pi_C$ and solving \cref{eq:cmpg-ca-cmdp} ensures that $\vpi^t\in\Pi_C$, $\forall t\leq T$. Therefore, $\vpi^{T-1}$ is feasible and is indeed an $\epsilon$-Nash policy.

\end{proof}

\section{Learning in Unknown Constrained Markov Potential Games (\cref{sec:learning-cmpgs})}
\label{app:learning-cmpgs}

This section provides the theoretical proofs for \cref{sec:learning-cmpgs}. As discussed in \cref{sec:learning-cmpgs}, in the learning setting, we need to modify \AlgNameShort~to make it work in the learning setting. The resulting algorithm \AlgNameShortExp~is described in \cref{alg:cmpg-ca}.
We prove the sample complexity bound for \AlgNameShortExp~in the following theorem.
\thmcmpgcasc*
\begin{proof}
First, we define the following events:
\begin{enumerate}
    \item $\G_{CMDP} = \curly{ \forall i\in [n], \forall t\in [T]: \text{$\hat{\pi}_i^t$ satisfies \cref{eq:sample-efficient-cmdp-solver}.}}$
    \item $\hat{\G}_{estimate} = \curly{ \forall i\in [n], \forall t\in [T]: \abs{\V^{r_i}(\hat{\pi}_i^t, \vpi^{t-1}_{-i}) - \hat{\V}^{r_i}(\hat{\pi}_i^t, \vpi^{t-1}_{-i})} \leq \epsilon/8 }$
    \item $\G_{estimate} = \curly{ \forall i\in [n], \forall t\in [T]: \abs{\V^{r_i}(\vpi^{t-1}) - \hat{\V}^{r_i}(\vpi^{t-1})} \leq \epsilon/8 }$
\end{enumerate}
Then, we define the "good event" as $\G \triangleq \G_{CMDP} \cap \hat{\G}_{estimate} \cap \G_{estimate}$. In the rest of the proof, we will use the notation $\mathcal{E}^C$ to refer to the complement of an event $\mathcal{E}$.

\textbf{1. For any $t\leq T$, if the algorithm did not terminate, the potential function is {\em strictly} increased:}
Assume that $\G$ holds. For any $t\leq T$ before termination, there exists an agent $j\in [n]$, for which $\epsilon_j^t>\epsilon/2$ holds. This implies the following increase in the potential function:
\begin{align*}
    \mPhi(\hat{\pi}_j^t, \vpi_{-j}^{t-1}) - \mPhi(\vpi^{t-1}) &= \V^{r_j}(\hat{\pi}_j^t, \vpi_{-j}^{t-1}) - \V^{r_j}(\vpi^{t-1}) \tag{By \cref{eq:potential-property}.}
    \\
    &\geq \paren{\hat{\V}^{r_j}(\hat{\pi}_j^t, \vpi_{-j}^{t-1}) - \frac{\epsilon}{8}} - \paren{ \hat{\V}^{r_j}(\vpi^{t-1})  + \frac{\epsilon}{8} } \tag{By $\G_{estimate}, \hat{\G}_{estimate}$.} \\
    &= \epsilon_j^t - \frac{\epsilon}{4} \\
    &> \frac{\epsilon}{4} \tag{Since $\epsilon_j^t>\epsilon/2$.}
\end{align*}

\textbf{2. $\vpi^T$ is an $\epsilon$-Nash policy:}
Assume that $\G$ holds. We apply \cref{lem:potential-bounded} with the base policy $\vpi^0$ to establish that $\mPhi(\vpi^T) - \mPhi(\vpi^0) \leq nH$. Since $\mPhi$ is increased by {\em at least} $\epsilon/4$ in every iteration $t\leq T$, the condition $\max_{i\in [n]}\epsilon_i^T \leq \epsilon/2$ must hold at time $T$. With this, we can bound the differences in the {\em true} value functions as follows, for every agent $i\in [n]$:
\begin{align*}
    {\V}^{r_i}(\hat{\pi}_i^T, \vpi_{-i}^{T-1}) - {\V}^{r_i}(\vpi^{T-1})
    &\leq \paren{ \hat{\V}^{r_i}\paren{\hat{\pi}_i^T, \vpi_{-i}^{T-1}} + \epsilon/8} - \paren{ \hat{\V}^{r_i}\paren{\hat{\pi}_i^T, \vpi_{-i}^{T-1}} - \epsilon/8} \tag{By $
    G_{estimate}, \hat{\G}_{estimate}$.} \\
    &= \epsilon_i^T + \frac{\epsilon}{4} \\
    &\leq \frac{3}{4} \epsilon.
\end{align*}
Combining this with $\G_{CMDP}$, we obtain:
\begin{align*}
    \max_{\pi \in \Pi_C^i(\vpi_{-i}^{T-1})} \V^{r_i}(\pi, \vpi_{-i}^{T-1}) - \V^{r_i}(\vpi^{T-1}) &= \max_{\pi \in \Pi_C^i(\vpi_{-i}^{T-1})} \V^{r_i}(\pi, \vpi_{-i}^{T-1}) - \V^{r_i}(\hat{\pi}_i^T, \vpi_{-i}^{T-1}) + \\
    &\quad\quad \V^{r_i}(\hat{\pi}_i^T, \vpi_{-i}^{T-1}) - \V^{r_i}(\vpi^{T-1}) \\
    &\leq \epsilon.
\end{align*}
Finally, observe that $\vpi^T \in \Pi_C$ and therefore, $\vpi^T$ is an $\epsilon$-Nash policy. Until now, we focused on the convergence of \AlgNameShortExp, {\em assuming} that the event $\G$ holds. In the following paragraphs, we bound the number of samples that are required for $\G$ to hold with high probability.

\textbf{3. Number of samples to obtain $Pr\bracket{\G_{CMDP}} \geq 1-\frac{\delta}{2}$:}
Define $\delta' \triangleq \frac{\epsilon \delta}{8 n^2 H}\in (0,1)$. Suppose that in iteration $t\leq T$, agent $i\in [n]$ uses a CMDP solver with a sample complexity of $\F_C\paren{|\StateSpace|, |\Action_i|, H, \zeta_i^t, \delta', \frac{\epsilon}{4}}$ s.t. $\hat{\pi}_i^t$ satisfies \cref{eq:sample-efficient-cmdp-solver} with probability at least $1-\delta'$. We take a union bound over all iterations $t\in \curly{1,...,T}$ and all agents $i\in [n]$, and obtain the following bound on $Pr\bracket{\G_{CMDP}^C}$:
\begin{align*}
    Pr\bracket{\G_{CMDP}^C } &\leq \sum_{i=1}^n \sum_{t=1}^T Pr\bracket{ \hat{\pi}_i^t \text{ does not satisfy \cref{eq:sample-efficient-cmdp-solver}}} \tag{Union bound.}\\
    &\leq nT \delta' \\
    &= \frac{\delta}{2}. \\
    \Rightarrow  Pr\bracket{\G_{CMDP} } &= 1-  Pr\bracket{\G_{CMDP}^C } \geq 1-\frac{\delta}{2}.
\end{align*}
Finally, to obtain $Pr\bracket{\G_{CMDP} } \geq 1-\frac{\delta}{2}$, we require the following number of samples:
\begin{equation*}
    \sum_{t=1}^T \sum_{i=1}^n \F_C\paren{|S|, |\Action_i|, H, \zeta_i^t, \frac{\epsilon \delta}{8 n^2 H}, \frac{\epsilon}{4}}.
\end{equation*}

\textbf{4. Number of samples to obtain $Pr\bracket{\G_{estimate} \cap \hat{\G}_{estimate}} \geq 1-\frac{\delta}{2}$:} Consider an arbitrary policy $\vpi\in\Pi$ and suppose that the agents execute $\vpi$ for $\M>0$ episodes.
Then, each agent $i\in [n]$ estimates $\hat{\V}^{r_i}(\vpi)$ with the averaged, cumulative reward from those $M$ episodes. Since the episodes are {\em independent}\footnote{After every episode, the initial state for the next episode is freshly sampled from $\mu$.} from each other, and the cumulative reward per episode is bounded in the range $[0,H]$, we can apply Hoeffding's inequality to bound the estimation error for agent $i$ as follows:
\begin{align*}
    Pr\bracket{ \abs{ \hat{\V}^{r_i}(\vpi) - \V^{r_i}(\vpi) } \geq \frac{\epsilon}{8} } &\leq 2 \exp \paren{\frac{-2 M \paren{\epsilon/8}^2 }{H^2} } \tag{By Hoeffding's inequality.} \\
    &= 2 \exp \paren{\frac{- M \epsilon^2}{32 H^2}}.
\end{align*}
The agents need to estimate $2 nT$ value functions in total for $\G_{estimate}$ and $\hat{\G}_{estimate}$. By setting $M = \frac{32 H^2}{\epsilon^2} \log \paren{\frac{32 n^2 H}{\epsilon \delta}}$, we ensure that $Pr\bracket{ \abs{ \hat{\V}^{r_i}(\vpi) - \V^{r_i}(\vpi) } \geq \frac{\epsilon}{8} } \leq \frac{\epsilon \delta}{16 n^2 H}$ holds, for all policies $\vpi$ that are required for $\G_{estimate}$ and $\hat{\G}_{estimate}$. We are now ready to bound $Pr\bracket{ \G_{estimate}^C \cup \hat{\G}_{estimate}^C }$:
\begin{align*}
    Pr\bracket{ \G_{estimate}^C \cup \hat{\G}_{estimate}^C } &\leq Pr\bracket{ \G_{estimate}^C} + Pr\bracket{ \hat{\G}_{estimate}^C } \tag{Union bound.}\\
    &\leq \sum_{t=1}^T \sum_{i=1}^n Pr\bracket{ \abs{ \hat{\V}^{r_i}(\vpi^{t-1}) - \V^{r_i}(\vpi^{t-1}) } \geq \frac{\epsilon}{8}} + \\
    &\quad\quad  \sum_{t=1}^T \sum_{i=1}^n Pr\bracket{ \abs{ \hat{\V}^{r_i}(\hat{\pi}_i^t, \vpi_{-i}^{t-1}) - \V^{r_i}(\hat{\pi}_i^t, \vpi_{-i}^{t-1}) } \geq \frac{\epsilon}{8} } \tag{Union bound.}\\
    &\leq  \frac{8 n^2 H}{\epsilon} \cdot \frac{\epsilon \delta}{16 n^2 H} \\
    &\leq \frac{\delta}{2}. \\
    \Rightarrow Pr\bracket{ \G_{estimate} \cap \hat{\G}_{estimate}} &= 1- Pr\bracket{ \G_{estimate}^C \cup \hat{\G}_{estimate}^C } \geq 1-\frac{\delta}{2}.
\end{align*}

To obtain $Pr\bracket{\G_{estimate} \cap \hat{\G}_{estimate}} \geq 1-\frac{\delta}{2}$, considering that each episode requires $H$ samples, the number of samples required is at most:
\begin{equation*}
2 n T M H \leq \frac{256 n^2 H^4}{\epsilon^3} \log \paren{\frac{32 n^2 H}{\epsilon \delta}}.
\end{equation*}

\textbf{5. Conclusion:}
Combining 3. and 4., we obtain $Pr[\G]\geq 1-\delta$ with the following number of samples:
\begin{equation*}
    \sum_{t=1}^T \sum_{i=1}^n \F_C\paren{|S|, |\Action_i|, H, \zeta_i^t, \frac{\epsilon \delta}{8 n^2 H}, \frac{\epsilon}{4}} + \frac{256 n^2 H^4}{\epsilon^3} \log \paren{\frac{32 n^2 H}{\epsilon \delta}}.
\end{equation*}
\end{proof}

In \cref{thm:cmpg-ca-sample-complexity}, we proved the sample complexity bound for \AlgNameShortExp, however, the actual number of samples will depend on the CMDP solver that is used within \AlgNameShortExp. In the following two sections, we will instantiate \AlgNameShortExp~with two different state-of-the-art CMDP solvers and state the resulting sample complexity bounds. Recall that both algorithms are designed for CMDPs with a {\em single} constraint. Due to this, we set $k=1$ and denote our cost function by $\curly{ c_h }_{h=1}^H$ and refer to the constraint parameter as $\alpha$. Note that this is due to a limitation of the existing CMDP algorithms and not of \AlgNameShortExp.
\section{Learning in Unknown Constrained Markov Potential Games: Generative Model (\cref{subsec:cmdp-generative-model})}
\label{app:cmdp-generative-model}

In this section, similar to \citet{vaswani2022nearoptimal}, we present a novel algorithm for finite-horizon CMDPs, assuming access to a generative model. While the algorithm by \citet{vaswani2022nearoptimal} returns a {\em mixture} policy\footnote{A mixture policy is a probability distribution over a set of policies.}, we require an {\em actual} policy. To address this issue, we make use of {\em occupancy measures}, which we discuss this in detail in \cref{subsec:occupancy-measures}. Our CMDP algorithm is described in \cref{alg:cmdp-generative-model}.
We start by proving the sample complexity bound for \cref{alg:cmdp-generative-model}. Before we start, we first introduce the following definitions, which are used in the proofs in this section:
\begin{align}
    {\pi}^* \in \arg\max_{\pi\in\Pi_C} {\V}^r(\pi), \text{ subject to: } {\V}^c(\pi) \leq \alpha \label{eq:cmdp-gen-optimal-policy} \tag{Optimal policy} \\
    \hat{\pi}^* \in \arg\max_{\pi\in\Pi} \hat{\V}^r(\pi), \text{ subject to: } \hat{\V}^c(\pi) \leq \alpha' \label{eq:cmdp-gen-empirical-optimal-policy} \tag{Optimal empirical policy} \\
    \lambda^* \text{ such that: }\hat{\V}^r(\hat{\pi}^*) = \max_{\pi\in\Pi}\bracket{\hat{\V}^r(\pi) + \lambda^* \paren{\alpha' - \hat{\V}^c(\pi)}} \tag{Optimal dual variable, \citep{paternain2019constrained}.} \label{eq:cmdp-gen-empirical-dual-variable},
\end{align}
where the value function $\hat{\V}^l$ is defined with respect to the empirical transition model $\hat{\Transition}$ (\cref{algln:cmdp-gen-model-sampling}), for $l=r,c$. We are now ready to state the sample complexity bound for \cref{alg:cmdp-generative-model} in the following lemma.
\begin{restatable}{lemma}{lemcmdpscgen}
\label{lem:cmdp-sample-complexity-gen-model}
    Consider a CMDP $\M = \paren{\StateSpace,\Action,H,\curly{\reward_h}_{h\in[H]}, \curly{\Transition_h}_{h\in [H]}, \curly{c_h}_{h\in [H]}, \alpha}$. Assume that the CMDP is {\em strictly} feasible, i.e., $\exists \pi\in\Pi_C$ s.t. $\V^c(\pi)<\alpha$. Furthermore, assume that the Slater constant $\zeta\triangleq \max_{\pi\in\Pi}\curly{\alpha - \V^c(\pi)}$, $\zeta>0$, is known. Then, for any $\epsilon'>0$ and $\delta' \in (0,1)$, \cref{alg:cmdp-generative-model} with parameters $N= \wtilde{\O} \paren{\frac{H^6 \log \paren{\frac{1}{\delta'}}}{\epsilon'^2 \zeta^2}}$, $\alpha' = \alpha - \frac{\epsilon' \zeta}{16 H}$, $U = \frac{8H}{\zeta}$, $T=\O \paren{\frac{H^6}{\zeta^4 \epsilon'^2}}$ and $\eta = \frac{U}{\sqrt{T} H}$ returns a policy $\bar{\pi}\in\Pi_C$ s.t. $\max_{\pi \in \Pi_C} \V^{r}(\pi) - \V^r(\bar{\pi}) \leq \epsilon'$, with probability at least $1-\delta'$. Thus, the sample complexity of \cref{alg:cmdp-generative-model} is $\F_{C} \paren{|\StateSpace|,|\Action_i|,\Horizon,\zeta,\delta',\epsilon'} = |\StateSpace||\Action| H N = \wtilde{\O} \paren{\frac{|\StateSpace||\Action|H^7 \log \paren{\frac{1}{\delta'}} }{\epsilon'^2 \zeta^2}}$.
\end{restatable}
\begin{proof}
We prove the lemma for an arbitrary $\Delta>0$ and constraint threshold $\alpha'\triangleq\alpha-\Delta$. We denote by $\epsilon_{opt}$ the error from the primal-dual algorithm (\cref{algln:cmdp-gen-model-primal-dual}-\cref{algln:cmdp-gen-model-primal-dual-end}). As shown in \cref{lem:primal-dual}, the policies $\curly{\hat{\pi}_t}_{t=0}^{T-1}$ satisfy:
\begin{align*}
    \hat{\V}^r(\hat{\pi}^*) - \frac{1}{T}\sum_{t=0}^{T-1} \hat{\V}^r(\hat{\pi}_t) &\leq \epsilon_{opt}, \\
    \frac{1}{T} \sum_{t=0}^{T-1} \hat{\V}^c(\hat{\pi}_t) &\leq \alpha'+\epsilon_{opt}.
\end{align*}
For now, we restrict $\epsilon_{opt}<\Delta$ and later determine the right choice of $\epsilon_{opt}$ and $\Delta$. Next, we analyze the reward sub-optimality and constraint violation guarantees.

\textbf{Constraint Violation:} Assume that $\abs{\V^c(\bar{\pi}) - \hat{\V}^c(\bar{\pi})} \leq \Delta - \epsilon_{opt}$ holds. We will later determine how many samples are required in \cref{algln:cmdp-gen-model-sampling} to satisfy this. Then, it is easy to verify that $\bar{\pi}$ does indeed satisfy the constraints in the original CMDP $\M$, i.e.
\begin{align*}
    \V^c(\bar{\pi}) &\leq \hat{\V}^c(\bar{\pi}) + \Delta - \epsilon_{opt} \\
    &\leq \paren{\alpha' + \epsilon_{opt}} - \Delta + \epsilon_{opt} \tag{By \cref{lem:primal-dual}.} \\
    &= \alpha \tag{Since $\alpha'=\alpha+\Delta$.}
\end{align*}

\textbf{Reward Sub-Optimality:} First, we define the following {\em relaxed} objective with respect to the empirical CMDP $\hat{\M}$:
\begin{equation}
\label{eq:empirical-cmdp-relaxed}
    \wtilde{\pi}^* \in \arg\max_{\pi\in\Pi} \hat{\V}^r(\pi), \text{ subject to: } \hat{\V}^c(\pi)\leq \alpha+\Delta.
\end{equation}
Next, we decompose the reward sub-optimality, by adding and subtracting common terms, as follows:
\begin{align}
    \V^r(\pi^*) - \V^r(\bar{\pi}) &= \bracket{\V^r(\pi^*) - \hat{\V}^r(\pi^*)} + \bracket{\hat{\V}^r(\pi^*)- \hat{\V}^r(\hat{\pi}^*)} + \underbrace{\bracket{\hat{\V}^r(\hat{\pi}^*) - \hat{\V}^r(\bar{\pi})}}_{\leq \epsilon_{opt}, \text{ \cref{lem:primal-dual}}} + \nonumber \\
    &\quad\quad \bracket{\hat{\V}^r(\bar{\pi}) - \V^r(\bar{\pi})} \nonumber \\
    &\leq \underbrace{\bracket{\V^r(\pi^*) - \hat{\V}^r(\pi^*)}}_{\text{Concentration error}} + \underbrace{\bracket{\hat{\V}^r(\wtilde{\pi}^*)- \hat{\V}^r(\hat{\pi}^*)}}_{\text{Sensitivity error}} + \epsilon_{opt} +
    \underbrace{\bracket{\hat{\V}^r(\bar{\pi}) - \V^r(\bar{\pi})}}_{\text{Concentration error}} \label{eq:reward-decomposition}.
\end{align}
To bound the sensitivity error, first assume that $\abs{\V^c(\pi^*) - \hat{\V}^c(\pi^*)} \leq \Delta$ holds. Then, $\pi^*$ is feasible for \cref{eq:empirical-cmdp-relaxed}:
\begin{equation*}
    \hat{\V}^c(\pi^*) \leq \V^c(\pi^*) + \Delta  \leq \alpha+\Delta.
\end{equation*}
Since $\wtilde{\pi}^*$ is optimal for \cref{eq:empirical-cmdp-relaxed}, we know that $\hat{\V}^r(\pi^*) \leq \hat{\V}^r(\wtilde{\pi}^*)$ must hold. Then, we apply \cref{lem:sensitivity-error} to bound the sensitivity error as:
\begin{equation*}
    \hat{\V}^r(\wtilde{\pi}^*)- \hat{\V}^r(\hat{\pi}^*) = \abs{\hat{\V}^r(\wtilde{\pi}^*)- \hat{\V}^r(\hat{\pi}^*)} \leq 2\Delta \lambda^*.
\end{equation*}
To further bound $\lambda^*$, we define $\pi_c^* \in \arg\min_{\pi\in\Pi} \V^c(\pi)$ and assume that $\abs{\V^c(\pi_c^*) - \hat{\V}^c(\pi_c^*)} \leq \frac{\zeta}{2} - \Delta$ holds. Then, we can apply \cref{lem:bound-dual-variable} and obtain the bound $\lambda^* \leq \frac{2 H}{\zeta}$. Plugging this into \cref{eq:reward-decomposition}, we obtain:
\begin{equation*}
     \V^r(\pi^*) - \V^r(\bar{\pi}) \leq \bracket{\V^r(\pi^*) - \hat{\V}^r(\pi^*)} + \frac{4 \Delta H}{\zeta} + \epsilon_{opt} + \bracket{\hat{\V}^r(\bar{\pi}) - \V^r(\bar{\pi})}.
\end{equation*}
Next, we discuss how to set $\Delta$ and $\epsilon_{opt}$. To achieve $\V^r(\pi^*) - \V^r(\bar{\pi}) \leq \epsilon$ in the end, we set $\Delta = \frac{\epsilon \zeta}{16 H}<\frac{\zeta}{2}$ and $\epsilon_{opt} = \frac{\Delta}{5} < \frac{\epsilon}{4}$\footnote{Here, we make a trade-off between sample and computational complexity. We only require $\epsilon_{opt}<\Delta$, but selecting a large $\epsilon_{opt}$ increases the sample complexity, due to the concentration bounds, whereas a smaller $\epsilon_{opt}$ increases the computational complexity in the primal-dual algorithm.}. This simplifies our reward sub-optimality to the following expression:
\begin{equation*}
     \V^r(\pi^*) - \V^r(\bar{\pi}) \leq \frac{\epsilon}{2}  + \bracket{\V^r(\pi^*) - \hat{\V}^r(\pi^*)} + \bracket{\hat{\V}^r(\bar{\pi}) - \V^r(\bar{\pi})}.
\end{equation*}
Based on our choice of $\epsilon_{opt}$, we set $U = \frac{8 H}{\zeta}$ (\cref{lem:primal-dual}). Note that \cref{lem:bound-dual-variable} guarantees that $U>\lambda^*$ then. Furthermore, we set $T=\frac{U^2 H^2}{\epsilon_{opt}^2 } \paren{1 + \frac{1}{(U-\lambda^*)^2}} = \O \paren{\frac{H^6}{\zeta^4 \epsilon^2}}$ (by \cref{lem:primal-dual}).

To obtain $\V^r(\pi^*) - \V^r(\bar{\pi})\leq \epsilon$, we require the following concentration bounds to hold overall:
\begin{align}
    \abs{\V^c(\pi^*) - \hat{\V}^c(\pi^*)} \leq \Delta, \quad \abs{\V^r(\pi^*) - \hat{\V}^r(\pi^*)} \leq \frac{\epsilon}{4},\quad \abs{\V^c(\pi_c^*) - \hat{\V}^c(\pi_c^*)} \leq 7 \Delta, \nonumber \\
    \abs{\V^c(\bar{\pi}) - \hat{\V}^c(\bar{\pi})} \leq \Delta - \epsilon_{opt}, \quad \abs{\V^r(\bar{\pi}) - \hat{\V}^r(\bar{\pi})} \leq \frac{\epsilon}{4}.
    \label{eq:cmdp-generative-concentration}
\end{align}
For the third inequality in \cref{eq:cmdp-generative-concentration}, note that $7\Delta \leq \frac{\zeta}{2} - \Delta$.

\textbf{Sample complexity:}
Applying \cref{lem:concentration-value-function} to each inequality in \cref{eq:cmdp-generative-concentration}, we have that with $N = \O \paren{ \frac{H^6 \log\paren{\frac{|\StateSpace||\Action|H}{\delta}}}{\epsilon^2 \zeta^2} } = \wtilde{\O}\paren{\frac{H^6 \log\paren{\frac{1}{\delta}} }{\epsilon^2 \zeta^2}}$, \cref{eq:cmdp-generative-concentration} holds with probability at least $1-\delta$. Thus, the sample complexity of \cref{alg:cmdp-generative-model} is $F_C(|\StateSpace|,|\Action|,H,\zeta,\delta,\epsilon) = |S||A|H N = \wtilde{\O} \paren{\frac{|\StateSpace||\Action|H^7 \log\paren{\frac{1}{\delta}}}{\epsilon^2 \zeta^2}}$.
\end{proof}

In the following section, we discuss how to construct a policy $\bar{\pi}$ s.t. $\hat{\V}^l(\bar{\pi}) = \frac{1}{T} \sum_{t=0}^{T-1} \hat{\V}^l(\hat{\pi}_t)$, for $l=r,c$.

\subsection{Convert mixture to a single policy}
\label{subsec:occupancy-measures}

Value functions are typically not linear in the policy, i.e., $\frac{1}{T}\sum_{t=0}^{T-1} \hat{\V}^l(\hat{\pi}_t)$ is, in general, not equal to $\hat{\V}^l\paren{\frac{1}{T}\sum_{t=0}^{T-1} \hat{\pi}_t}$, for $l=r,c$. State-action occupancy measures 
\citep{borkar1988convex}, however, allow us to reformulate the value function in a {\em linear} way. For brevity, we will refer to state-action occupancy measures as "occupancy measures" in the rest of this section.

\begin{definition}[Occupancy Measure]
\label{def:state-action-occupancy}
    For every policy $\pi$, we define its 
    {\em occupancy measure} $\curly{\hat{\rho}^{\pi}_h}_{h=1}^H$ with respect to the transition model $\curly{\hat{\Transition}_h}_{h=1}^H$ as follows:
    \begin{equation*}
        \hat{\rho}^{\pi}_h(s,a) \triangleq \sum_{(s,a)} Pr^{\pi}(s_h=s,a_h=a), \quad \forall (s,a,h) \in \StateAction \times [H].
    \end{equation*}
    Alternatively, the occupancy measure can be expressed {\em recursively} as follows:
    \begin{align*}
        \hat{\rho}^{\pi}_1(s,a) &= \mu(s)\cdot \pi_1(a|s), \\
        \hat{\rho}^{\pi}_h(s,a) &= \sum_{(s',a')} \hat{\rho}^{\pi}_{h-1}(s',a') \cdot \hat{\Transition}_h(s|s',a') \cdot \pi_h(a|s), \forall h\in [H]\setminus \curly{1},
    \end{align*}
\end{definition}
for all state-action pairs $(s,a) \in \StateAction$. The following lemma establishes how we can construct a policy $\bar{\pi}$ from the primal-dual policies $\curly{\hat{\pi}_t}_{t=0}^{T-1}$ s.t. $\hat{\V}^l(\bar{\pi}) = \frac{1}{T}\sum_{t=0}^{T-1} \hat{\V}^l(\hat{\pi}_t)$, for $l=r,c$.

\begin{lemma}
\label{lem:primaldual-policies-to-policy}
    Given policies $\curly{\hat{\pi}_t}_{t=0}^{T-1}$, consider the {\em averaged} occupancy measure $\bar{\rho} \triangleq \frac{1}{T} \sum_{t=0}^{T-1} \hat{\rho}^{\hat{\pi}_t}$. Define the policy $\bar{\pi}$ as follows:
    \begin{equation}
    \label{eq:avg-occ-to-policy}
        \bar{\pi}_h(a|s) \triangleq  \begin{cases} 
      \frac{\bar{\rho}_h(s,a)}{\sum_{a'} \bar{\rho}_h(s,a')} & \sum_{a'}\bar{\rho}_h(s,a')>0 \\
      \frac{1}{|\Action|} & \text{otherwise},
   \end{cases}
    \end{equation}
    $\forall h\in [H]$ and $\forall (s,a)\in\StateAction$. Then, $\bar{\pi}$ satisfies the following property:
    \begin{equation*}
        \hat{\V}^l(\bar{\pi}) = \frac{1}{T} \sum_{t=0}^{T-1} \hat{\V}^l(\hat{\pi}_t), \tag{for $l=r,c$.}.
    \end{equation*}
\end{lemma}
\begin{proof}
    Since the set of occupation measures is convex (\cref{lem:occupation-measure-convex}), $\bar{\rho}$ is a valid occupation measure, i.e., there exists a policy $\pi$ s.t. $\bar{\rho} \equiv \hat{\rho}^{\pi}$. Then, \cref{lem:occ-to-policy} guarantees that the policy $\bar{\pi}$ constructed in \cref{eq:avg-occ-to-policy} indeed satisfies $\hat{\rho}^{\bar{\pi}} \equiv \bar{\rho}$. We are now ready to prove the statement of this lemma:
    \begin{align*}
        \hat{\V}^l(\bar{\pi}) &= \sum_{h=1}^H \sum_{(s,a)} \hat{\rho}^{\bar{\pi}}_h(s,a) \cdot l_h(s,a) \tag{By \cref{lem:value-function-linear}.} \\
        &= \sum_{h=1}^H \sum_{(s,a)} \paren{\frac{1}{T} \sum_{t=0}^{T-1} \hat{\rho}^{\hat{\pi}_t}_h(s,a) } \cdot l_h(s,a) \\
        &= \frac{1}{T} \paren{\sum_{t=0}^{T-1} \underbrace{\sum_{h=1}^H \sum_{(s,a)} \hat{\rho}^{\hat{\pi}_t}_h(s,a) \cdot l_h(s,a)}_{=\hat{\V}^l(\hat{\pi}_t), \text{ \cref{lem:value-function-linear}}}} \\
        &= \frac{1}{T} \sum_{t=0}^{T-1} \hat{\V}^l(\hat{\pi}_t).
    \end{align*}
\end{proof}

\subsubsection{Standard lemmas for occupancy measures}

\textbf{Note: } The lemmas in this section are well-known in the literature \citep{efroni2020exploration}. We state them here for completeness.

\begin{lemma}
\label{lem:value-function-linear}
    Consider a policy $\pi\in\Pi$ and its respective occupancy measure $\curly{\hat{\rho}^{\pi}_h}_{h\in [H]}$ (\cref{def:state-action-occupancy}). Then, for any $l=r,c$, the value function $\hat{\V}^l(\pi)$ can be expressed in terms of the occupancy measure as follows:
    \begin{equation*}
        \hat{\V}^l(\pi) = \sum_{h=1}^H \sum_{(s,a)} \hat{\rho}^{\pi}_h(s,a) \cdot l_h(s,a).
    \end{equation*}
\end{lemma}
\begin{proof}
    See \citet[section 2.1]{efroni2020exploration}.
\end{proof}

\begin{lemma}[Convexity of Occupancy Measures]
\label{lem:occupation-measure-convex}
    The set of occupancy measures $\mathcal{D} \triangleq \curly{\curly{\hat{\rho}^{\pi}_h}_{h\in[H]} | \pi\in\Pi}$ is convex.
\end{lemma}
\begin{proof}
See \citet[section 2.3]{efroni2020exploration}.
\end{proof}

\begin{lemma}
\label{lem:occ-to-policy}
    Given a valid occupation measure $\curly{\hat{\rho}_h}_{h\in[H]}\in \mathcal{D}$, the following policy $\pi$ induces $\curly{\hat{\rho}_h}_{h\in[H]}$:
    \begin{equation}
        \label{eq:occ-to-policy}
        \pi_h(a|s) =  \begin{cases} 
      \frac{\rho_h(s,a)}{\sum_{a'}\rho_h(s,a')} & \sum_{a'}\rho_h(s,a')>0 \\
      \frac{1}{|\Action|} & \text{otherwise},
   \end{cases}
    \end{equation}
    $\forall h\in [H]$ and $\forall (s,a)\in\StateAction$.
\end{lemma}

\begin{proof}
    See \citet[section 2.3]{efroni2020exploration}.
\end{proof}

\subsection{Proofs of auxiliary lemmas}

\begin{lemma}[Guarantees for the primal-dual algorithm]
\label{lem:primal-dual}
For any $\epsilon_{opt}>0$, with $U >\lambda^*$, $T = \frac{U^2 H^2}{\epsilon_{opt}^2} \paren{1+\frac{1}{(U-\lambda^*)^2}}$, $\eta = \frac{U}{\sqrt{T} H}$, the primal-dual algorithm (\cref{algln:cmdp-gen-model-primal-dual} - \cref{algln:cmdp-gen-model-primal-dual-end}) produces policies $\curly{\hat{\pi}_t}_{t=0}^{T-1}$, which satisfy:
\begin{align}
    \hat{\V}^r(\hat{\pi}^*) - \frac{1}{T} \sum_{t=0}^{T-1} \hat{\V}^r(\hat{\pi}_t) &\leq \epsilon_{opt}, \label{eq:primal-dual-reward-subopt} \\
    \quad \frac{1}{T} \sum_{t=0}^{T-1}\hat{\V}^c(\hat{\pi}_t) &\leq \alpha'+\epsilon_{opt}. \label{eq:primal-dual-constraint-violation}
\end{align}
\end{lemma}
\begin{proof}
In iteration $t\in \curly{0,...,T-1}$, we compute $\hat{\pi}_t \in \arg\max_{\pi\in\Pi} \hat{\V}^{r-\lambda_t c}(\pi)$.\footnote{This is an {\em unconstrained} Markov decision process (MDP) and can be solved {\em exactly} using dynamic programming \citep{bertsekas1995dynamic}}. Since $\hat{\pi}_t$ is optimal at time $t$, the following inequality holds for every $\pi\in\Pi$:
\begin{equation*}
    \hat{\V}^{r}(\hat{\pi}_t) + \lambda_t \paren{\alpha' - \hat{\V}^c(\hat{\pi}_t)} \geq \hat{\V}^r(\pi) + \lambda_t \paren{\alpha' - \hat{\V}^c(\pi)}
\end{equation*}
Setting $\pi = \hat{\pi}^*$, and using the fact that $\hat{\pi}^*$ is feasible with respect to $\hat{\M}$ i.e., $\hat{\V}^c(\hat{\pi}^*) \leq \alpha'$, we obtain:
\begin{equation*}
    \hat{\V}^r(\hat{\pi}^*) - \hat{\V}^r(\hat{\pi}_t) \leq \lambda_t \paren{\alpha' - \hat{\V}^c(\hat{\pi}_t)}.
\end{equation*}
Fixing $\lambda \in [0,U]$, subtracting $\lambda \paren{\alpha' - \hat{\V}^c(\hat{\pi}_t)}$, summing over all $t=0,...,T-1$, and dividing by $T$, we obtain:
\begin{align}
    \hat{\V}^r(\hat{\pi}^*) - \frac{1}{T}\sum_{t=1}^T \hat{\V}^r(\hat{\pi}_t) - \lambda \paren{\alpha' - \frac{1}{T}\sum_{t=1}^T \hat{\V}^c(\hat{\pi}_t)} &\leq \frac{1}{T}\sum_{t=0}^{T-1} \paren{\lambda_t - \lambda} \paren{\alpha' - \hat{\V}^c(\hat{\pi}_t) } \nonumber \\
    &\leq \frac{U H}{\sqrt{T}} \quad{\text{(By \cref{lem:cmdp-dual-regret}, setting $\eta = \frac{U}{\sqrt{T} H}$.)}} \label{eq:primal-dual-regret}
\end{align}
We will now bound the reward sub-optimality (\cref{eq:primal-dual-reward-subopt}) and constraint violation (\cref{eq:primal-dual-constraint-violation}) separately.

\textbf{Reward Sub-Optimality:} Since \cref{eq:primal-dual-regret} holds for any $\lambda\in [0,U]$, we can set $\lambda=0$ to obtain:
\begin{equation}
\label{eq:reward-suboptimality-primal-dual}
    \hat{\V}^r(\hat{\pi}^*) - \frac{1}{T} \sum_{t=0}^{T-1}\hat{\V}^r(\hat{\pi}_t) \leq \frac{U H}{\sqrt{T}}.
\end{equation}

\textbf{Constraint Violation:} If $\frac{1}{T} \sum_{t=0}^{T-1} \hat{\V}^c(\hat{\pi}_t)\leq \alpha'$, then \cref{eq:primal-dual-constraint-violation} holds trivially. Consider the case $\frac{1}{T} \sum_{t=0}^{T-1} \hat{\V}^c(\hat{\pi}_t) > \alpha'$. Recall that $\bar{\pi}$ is constructed in a way that it satisfies $\hat{\V}^l(\bar{\pi}) = \frac{1}{T} \sum_{t=0}^{T-1} \hat{\V}^l(\hat{\pi}_t)$ for $l=r,c$ (\cref{algln:cmdp-gen-model-convert}). Plugging this into \cref{eq:primal-dual-regret} with $\lambda=U \in [0,U]$, we obtain:
\begin{align}
     \hat{\V}^r(\hat{\pi}^*) - \hat{\V}^r(\bar{\pi}) + U \underbrace{\paren{ \hat{\V}^c(\bar{\pi}) - \alpha'}}_{\geq 0} &\leq \frac{U H}{\sqrt{T}} \nonumber \\
     \Rightarrow  \hat{\V}^c(\bar{\pi}) - \alpha' &\leq  \frac{U H}{\sqrt{T} (U-\lambda^*)}. \tag{By \cref{lem:constraint-violation-positive}.} \\
     \Rightarrow  \frac{1}{T} \sum_{t=0}^{T-1}\hat{\V}^c(\hat{\pi}_t) - \alpha' &\leq \frac{U H}{\sqrt{T} (U-\lambda^*)}. \label{eq:constraint-violation-primal-dual}
\end{align}
Finally, we set $T = \frac{U^2 H^2}{\epsilon_{opt}^2} \paren{1+\frac{1}{(U-\lambda^*)^2}}$ s.t. \cref{eq:reward-suboptimality-primal-dual} and \cref{eq:constraint-violation-primal-dual} are both bounded by $\epsilon_{opt}$. This completes the proof of this lemma.
\end{proof}

\begin{lemma}[Dual regret]
\label{lem:cmdp-dual-regret}
For any $\lambda\in [0,U]$, the dual regret $R^d(\lambda,T) \triangleq \sum_{t=0}^{T-1} \paren{\lambda_t-\lambda} \paren{\alpha' - \hat{\V}^{c}(\hat{\pi}_t)}$ can be bounded as follows:
\begin{equation*}
    R^d(\vlambda,T)  \leq U \sqrt{T} H,
\end{equation*}
by setting $\eta = \frac{U}{\sqrt{T} H}$ in \cref{alg:cmdp-generative-model}.
\end{lemma}
\begin{proof}

For any $t\in \curly{0,...,T-1}$, consider the term $\abs{\lambda_{t+1}-\lambda}^2$:
\begin{align*}
    \abs{\lambda_{t+1} - \lambda}^2 &= \abs{ \P_{[0,U]}\bracket{\lambda_t - \eta \paren{\alpha' - \hat{\V}^c(\hat{\pi}_t)}} - \lambda}^2 \\
    &\leq \abs{ \lambda_t - \eta \paren{\alpha' - \hat{\V}^c(\hat{\pi}_t)} - \lambda }^2 \tag{The projection $\P_{[0,U]}$ is non-expansive.} \\
    &= \abs{\lambda_{t} - \lambda}^2 - 2 \eta \paren{\lambda_t-\lambda} \paren{\alpha' - \hat{\V}^c(\hat{\pi}_t)} + \eta^2 \underbrace{\abs{\alpha' - \hat{\V}^c(\hat{\pi}_t) }^2}_{\leq H^2} \\
    &\leq \abs{\lambda_{t} - \lambda}^2 - 2 \eta \paren{\lambda_t-\lambda} \paren{\alpha' - \hat{\V}^c(\hat{\pi}_t)} + \eta^2 H^2.
\end{align*}
Rearranging the terms, dividing by $2\eta$ and summing over $t=0,...,T-1$, we obtain a bound on the dual regret:
\begin{align*}
    R^d(\lambda, T) &= \sum_{t=0}^{T-1} \paren{\lambda_t - \lambda} \paren{\alpha' - \hat{\V}^{c}(\hat{\pi}_t)} \\
    &\leq \frac{1}{2\eta}\sum_{t=0}^{T-1} \paren{\abs{\lambda_t - \lambda}^2 - \abs{\lambda_{t+1}-\lambda}^2} + \frac{\eta T H^2}{2} \\
    &= \frac{\overbrace{\abs{\lambda_0 - \lambda}^2}^{\leq U^2} - \overbrace{\abs{\lambda_T - \lambda}^2}^{\geq 0}}{2 \eta} + \frac{\eta T H^2}{2} \\
    &\leq U \sqrt{T} H. \tag{Setting $\eta = \frac{U}{\sqrt{T} H}$.}
\end{align*}
\end{proof}

\begin{lemma}
\label{lem:constraint-violation-positive}
For any $C>\lambda^*$ and any policy $\wtilde{\pi}$ s.t. $\hat{\V}^{r}(\hat{\pi}^*) - \hat{\V}^r(\wtilde{\pi}) + C \bracket{\hat{\V}^{c}(\wtilde{\pi}) - \alpha'}_+\leq \beta$, we have:
\begin{equation*}
    \bracket{\hat{\V}^{c}(\wtilde{\pi}) - \alpha'}_+ \leq \frac{\beta}{C-\lambda^{*}}.
\end{equation*}
\end{lemma}
\begin{proof}
We define the function $\nu(\tau) \triangleq \max_{\pi\in\Pi}\curly{\hat{\V}^r(\pi) \Big| \hat{\V}^c(\pi) \leq \alpha' - \tau}$ for any $\tau\in\Reals$. Strong duality for CMDPs \citep{paternain2019constrained} gives us the following inequality:
\begin{equation*}
    \nu(0) = \hat{\V}^r(\hat{\pi}^*) = \max_{\pi\in\Pi} \bracket{\hat{\V}^r(\pi) + \lambda^* \paren{\alpha' - \hat{\V}^c(\pi)}}.
\end{equation*}
Next, consider an arbitrary $\tau$ and any policy $\pi'$ s.t. $\hat{\V}^c(\pi') \leq \alpha' - \tau$. Subtracting $\tau \lambda^*$ from the inequality above, we obtain:
\begin{align*}
    \nu(0) - \tau \lambda^* &= \max_{\pi\in\Pi} \bracket{\hat{\V}^r(\pi) + \lambda^* \paren{\alpha' - \hat{\V}^c(\pi)}} - \tau \lambda^* \\
    &\geq \hat{\V}^r(\pi') + \lambda^* \paren{\alpha' - \hat{\V}^c(\pi')} - \tau \lambda^* \\
    &= \hat{\V}^r(\pi') + \lambda^* \underbrace{\paren{\alpha' - \tau - \hat{\V}^c(\pi')}}_{\geq 0} \\
    &\geq \hat{\V}^r(\pi').
\end{align*}
Since this inequality holds for any $\pi'$ with $\hat{\V}^c(\pi')\leq \alpha'-\tau$, it also holds for $\pi'^* \in \arg\max_{\pi\in\Pi} \curly{\hat{\V}^r(\pi) \Big| \hat{\V}^c(\pi) \leq \alpha'-\tau}$. Plugging this into the inequality above, we obtain:
\begin{align*}
    \nu(0) - \tau \lambda^* &\geq \underbrace{\max_{\pi \in \Pi}\curly{\hat{\V}^r(\pi) \Big| \hat{\V}^c(\pi) \leq \alpha' - \tau}}_{=\nu(\tau)} \\
    \Rightarrow \tau \lambda^* &\leq \nu(0) - \nu(\tau).
\end{align*}
Now, we select $\wtilde{\tau} = -\paren{\hat{\V}^c(\wtilde{\pi}) - \alpha'}$ and bound the following expression:
\begin{align*}
    \paren{C-\lambda^*}\abs{\wtilde{\tau}} &=  \wtilde{\tau}\lambda^* + C \abs{\wtilde{\tau}} \\
    &\leq \nu(0) - \nu(\wtilde{\tau}) + C\abs{\wtilde{\tau}} \\
    &= \underbrace{\hat{\V}^r(\hat{\pi}^*) - \hat{\V}^r(\wtilde{\pi}) + C \abs{\wtilde{\pi}}}_{=\beta} + \underbrace{\hat{\V}^r(\wtilde{\pi}) - \nu(\wtilde{\tau})}_{\leq 0} \tag{Addition and subtraction of common terms.} \\
    \Rightarrow \abs{\wtilde{\tau}} = \bracket{ \hat{\V}^c(\wtilde{\pi}) - \alpha'}_+ &\leq \frac{\beta}{C-\lambda^*}.
\end{align*}

\end{proof}

\begin{lemma}[Sensitivity Error]
\label{lem:sensitivity-error}
     Let $\Delta>0$ and define $\hat{\pi}^*, \wtilde{\pi}^*$ as follows:
     \begin{align}
         \hat{\pi}^* &\in \arg\max_{\pi\in\Pi} \bracket{\hat{\V}^r(\pi), \text{ subject to: } \hat{\V}^{c}(\pi)\leq \alpha - \Delta} \label{eq:sensitivity-error-tight}\\
         \wtilde{\pi}^* &\in \arg\max_{\pi\in\Pi} \bracket{\hat{\V}^r(\pi), \text{ subject to: } \hat{\V}^{c}(\pi)\leq \alpha + \Delta}  \nonumber
     \end{align}
     Then, the sensitivity error $\abs{\hat{\V}^r(\hat{\pi}^*) - \hat{\V}^r(\wtilde{\pi}^*)}$ can be bounded as follows:
     \begin{equation*}
         \abs{\hat{\V}^r(\hat{\pi}^*) - \hat{\V}^r(\wtilde{\pi}^*)} \leq 2\Delta \lambda^*,
     \end{equation*}
     where $\lambda^*$ is the optimal dual variable for \cref{eq:sensitivity-error-tight}, i.e.:
     \begin{equation*}
         \hat{\V}^r(\hat{\pi}^*) = \max_{\pi\in \Pi}\bracket{\hat{\V}^r(\pi) + \lambda^* \paren{\alpha - \Delta - \hat{\V}^c(\pi)}},
     \end{equation*}
     which holds due to the strong duality property for CMDPs \citep{paternain2019constrained}.
 \end{lemma}
 
 \begin{proof}
We start with the strong duality property for $\hat{\pi}^*$:
   \begin{align*}
        \hat{\V}^r(\hat{\pi}^*) &= \max_{\pi\in\Pi} \bracket{\hat{\V}^{r}(\pi) +  \lambda^{*} \paren{\alpha - \Delta - \hat{\V}^{c}(\pi)}} \\
        &\geq \hat{\V}^r(\wtilde{\pi}^*) + \lambda^{*} \paren{\alpha - \Delta - \underbrace{\hat{\V}^{c}(\wtilde{\pi}^*)}_{\leq \alpha+\Delta}} \\
        &\geq \hat{\V}^r(\wtilde{\pi}^*) - 2 \Delta \lambda^* \\
        &\Rightarrow \hat{\V}^r(\wtilde{\pi}^*) - \hat{\V}^r(\hat{\pi}^*) \leq 2\Delta \lambda^*.
    \end{align*}

Note that the constraint set considered for $\wtilde{\pi}^*$ is larger than the one for $\hat{\pi}^*$. Therefore, $\hat{\V}^r(\wtilde{\pi}^*) \geq \hat{\V}^r(\hat{\pi}^*)$ holds and this implies:
\begin{equation}
    \abs{\hat{\V}^r(\wtilde{\pi}^*) - \hat{\V}^r(\hat{\pi}^*)} = \hat{\V}^r(\wtilde{\pi}^*) - \hat{\V}^r(\hat{\pi}^*) \leq  2\Delta \lambda^*.
\end{equation}
\end{proof}

\begin{lemma}[Bound on the dual variable]
\label{lem:bound-dual-variable}

Define $\pi_c^* \triangleq \arg\min_{\pi\in\Pi}\V^c(\pi)$ and $\zeta \triangleq \max_{\pi\in\Pi}\curly{\alpha - \V^c(\pi)}$. Let $\alpha'=\alpha-\Delta$, for $\Delta\in \paren{0, \frac{\zeta}{2}}$ and assume that $\abs{\hat{\V}^c(\pi_c^*) - \V^c(\pi_c^*)}\leq \frac{\zeta}{2} - \Delta$ holds. Furthermore, let $\hat{\pi}^* \in \arg\max_{\pi\in\Pi}\bracket{\hat{\V}^r(\pi) \text{ subject to: } \hat{\V}^c(\pi)\leq \alpha'}$ denote the optimal policy to the empirical CMDP, and let $\lambda^*$ denote the corresponding optimal dual variable, i.e.:
\begin{equation*}
    \hat{\V}^r(\hat{\pi}^*) = \max_{\pi \in \Pi} \bracket{\hat{\V}^r(\pi) + \lambda^* \paren{\alpha' - \hat{\V}^c(\pi)}}
\end{equation*}
Then, the dual variable $\lambda^*$ can be bounded as follows:
\begin{equation*}
    \lambda^* \leq \frac{2 H}{\zeta}.
\end{equation*}
\end{lemma}
\begin{proof}
First, we define the policy $\hat{\pi}_c^* \in \arg\min_{\pi\in\Pi}\hat{\V}^c(\pi)$. Next, we use the strong duality property \citep{paternain2019constrained} for $\hat{\pi}^*$:
\begin{align*}
    \hat{\V}^r(\hat{\pi}^*) &= \max_{\pi\in\Pi} \bracket{\hat{\V}^r(\pi) + \lambda^* \paren{\alpha' - \hat{\V}^c(\pi)}}\\
    &\geq \hat{\V}^r(\hat{\pi}_c^*) + \lambda^* \paren{\alpha' - \hat{\V}^c(\hat{\pi}_c^*)} \\
    &= \hat{\V}^r(\hat{\pi}_c^*) + \lambda^* \paren{ \underbrace{\alpha-\V^c(\pi_c^*)}_{=\zeta} - \Delta +  \V^c(\pi_c^*) - \hat{\V}^c(\hat{\pi}_c^*)} \\
    &= \hat{\V}^r(\hat{\pi}_c^*) + \lambda^* \paren{
    \zeta - \Delta + \underbrace{\hat{\V}^c(\pi_c^*) - \hat{\V}^c(\hat{\pi}_c^*)}_{\geq 0} + \underbrace{\V^c(\pi_c^*) - \hat{\V}^c(\pi_c^*)}_{\geq - \abs{\V^c(\pi_c^*) - \hat{\V}^c(\pi_c^*)}}
    } \\
    &\geq \hat{\V}^r(\hat{\pi}_c^*) + \lambda^* \paren{\zeta - \Delta -\underbrace{\abs{\hat{\V}^c(\pi_c^*) - \V^c(\pi_c^*)}}_{\leq \frac{\zeta}{2} - \Delta}} \\
    &\geq \hat{\V}^r(\hat{\pi}_c^*) + \frac{\lambda^* \zeta}{2} \\
    \Rightarrow \lambda^* &\leq \frac{2 \paren{\hat{\V}^r(\hat{\pi}^*) - \hat{\V}^r(\hat{\pi}_c^*)}}{\zeta}  \leq \frac{2 H}{\zeta}.
\end{align*}
\end{proof}

\begin{lemma}
\label{lem:concentration-value-function}
Given $\epsilon\in (0,H],\delta>0$, for each $(s,a,h) \in \StateAction\times[H]$, obtain $N = \frac{\log \paren{\frac{2 |\StateSpace|^2 |\Action| H}{\delta}} H^4}{\epsilon^2}$ independent samples from $\Transition_h(\cdot|s,a)$ and form the estimate $\hat{\Transition}_h(\cdot|s,a)$. Then, the following concentration bound holds for all policies $\pi$ and all $l\in\curly{r,c}$ uniformly with probability at least $1-\delta$:
\begin{equation*}
    \abs{\V^l(\pi) - \hat{\V}^l(\pi)} \leq \epsilon.
\end{equation*}
\end{lemma}
\begin{proof}
    We start with an arbitrary $\beta>0$ and assume that the difference between $\curly{\Transition_h}_{h\in [H]}$ and $\curly{\hat{\Transition}_h}_{h\in [H]}$ is bounded by $\beta$, i.e.
    \begin{equation}
    \label{eq:cmdp-gen-transition-concentration}
    \max_{(s,a,s',h)\in\StateAction\times\StateSpace\times[H]} \abs{\Transition_h(s'|s,a) - \hat{\Transition}_h(s'|s,a)} \leq \beta.
\end{equation}
For any policy $\pi \in\Pi$ and $l\in \curly{r,c}$, by \cref{lem:value-difference}, the difference $\V^l(\pi) - \hat{\V}^l(\pi)$ can be written as $\V^l(\pi) - \hat{\V}^l(\pi) = \mathop{\E}_{\substack{s \sim \mu,\\ a_h \sim \pi_h(\cdot | s_h),  \\s_{h+1} \sim {\hat{\Transition}_h}(\cdot | s_h, a_h) }}\bracket{\sum_{h=1}^{\Horizon} \paren{\Transition_h - \hat{\Transition}_h} \paren{s_{h+1}|s_h,a_h} \V_{h+1}^l(s_{h+1}; \pi) | s_0 = s }$, where $\V_{h+1}^l(\cdot;\pi)$ is defined according to \cref{eq:value-function-step}. Next, we bound $\abs{\V^l(\pi) - \hat{\V}^l(\pi)}$ as follows:
\begin{align*}
    \abs{\V^l(\pi) - \hat{\V}^l(\pi)} &= \abs{\mathop{\E}_{\substack{s \sim \mu,\\ a_h \sim \pi_h(\cdot | s_h),  \\s_{h+1} \sim {\hat{\Transition}_h}(\cdot | s_h, a_h) }}\bracket{\sum_{h=1}^{\Horizon} \paren{\Transition_h - \hat{\Transition}_h} \paren{s_{h+1}|s_h,a_h} \V_{h+1}^l(s_{h+1}; \pi) | s_0 = s }} \\
    &\leq \mathop{\E}_{\substack{s \sim \mu,\\ a_h \sim \pi_h(\cdot | s_h),  \\s_{h+1} \sim {{\Transition}_h}(\cdot | s_h, a_h) }} \bracket{\abs{\sum_{h=1}^H \paren{\Transition_h - \hat{\Transition}_h} \paren{s_{h+1}|s_h,a_h} \cdot \V_{h+1}^l(s_{h+1}; \pi) | s_0 = s }} \tag{Triangle inequality.} \\
    &\leq \mathop{\E}_{\substack{s \sim \mu,\\ a_h \sim \pi_h(\cdot | s_h),  \\s_{h+1} \sim {{\Transition}_h}(\cdot | s_h, a_h) }} \bracket{\sum_{h=1}^H \underbrace{\abs{\paren{\Transition_h - \hat{\Transition}_h} \paren{s_{h+1}|s_h,a_h}}}_{\leq \beta} \cdot \underbrace{\abs{\V_{h+1}^l(s_{h+1}; \pi)}}_{\leq H-h} | s_0 = s } \tag{Triangle inequality.} \\
    &\leq \beta H^2.
\end{align*}
To ensure that $\abs{\V^l(\pi) - \hat{\V}^l(\pi)} \leq \epsilon$, we set $\beta = \frac{\epsilon}{H^2}$. Now, consider an arbitrary $(s,a,h)$ and assume that we obtain $N>0$ {\em independent} samples from $\Transition_h(\cdot|s,a)$. Applying Hoeffding's inequality, we obtain the following bound on the estimation error:
\begin{equation}
    \label{eq:hoeffding}
    Pr\bracket{\abs{\Transition_h(s'|s,a) - \hat{\Transition}_h(s'|s,a)} > \beta} \leq 2 \exp \paren{-2N \beta^2}.
\end{equation}
Setting $N = \frac{\log \paren{\frac{2 |\StateSpace|^2 |\Action| H}{\delta}} H^4}{\epsilon^2}$, we obtain that $Pr\bracket{\abs{\Transition_h(s'|s,a) - \hat{\Transition}_h(s'|s,a)} > \beta} \leq \frac{\delta}{|\StateSpace|^2 |\Action| H}$. Taking the union bound over all $(s,a,s',h)$, we obtain \cref{eq:cmdp-gen-transition-concentration} with probability at least $1-\delta$.
\end{proof}

\begin{lemma}[Value difference lemma]
\label{lem:value-difference}
For any policy $\pi\in\Pi$ and $l\in \curly{r,c}$, the value difference $\V^l(\pi) - \hat{\V}^l(\pi)$ can be expressed as follows:
\begin{align*}
    \V^l(\pi) - \hat{\V}^l(\pi) &= \mathop{\E}_{\substack{s \sim \mu,\\ a_h \sim \pi_h(\cdot | s_h),  \\s_{h+1} \sim {\hat{\Transition}_h}(\cdot | s_h, a_h) }}\bracket{\sum_{h=1}^{\Horizon} \paren{\Transition_h - \hat{\Transition}_h} \paren{s_{h+1}|s_h,a_h} \V_{h+1}^l(\pi; s_{h+1}) | s_0 = s } \\
    &= \mathop{\E}_{\substack{s \sim \mu,\\ a_h \sim \pi_h(\cdot | s_h),  \\s_{h+1} \sim {{\Transition}_h}(\cdot | s_h, a_h) }}\bracket{\sum_{h=1}^{\Horizon} \paren{\hat{\Transition}_h - {\Transition}_h} \paren{s_{h+1}|s_h,a_h} \hat{\V}_{h+1}^l(\pi; s_{h+1}) | s_0 = s },
\end{align*}
where the per-step and per-state value functions are defined as follows:
\begin{align}
    \V_h^l(\pi;s) &\triangleq \mathop{\E}_{\substack{a_h \sim \pi_h(\cdot | s_h),  \\s_{h+1} \sim {{\Transition}_h}(\cdot | s_h, a_h) }}\bracket{\sum_{h'=h}^{\Horizon} l_h(s_h,a_h) | s_0 = s }, \forall s\in\StateSpace, \forall h\in [H], \label{eq:value-function-step}\\
    \hat{\V}_h^l(\pi;s) &\triangleq \mathop{\E}_{\substack{a_h \sim \pi_h(\cdot | s_h),  \\s_{h+1} \sim {\hat{\Transition}_h}(\cdot | s_h, a_h) }}\bracket{\sum_{h'=h}^{\Horizon} l_h(s_h,a_h) | s_0 = s }, \forall s\in\StateSpace, \forall h\in [H] \nonumber.
\end{align}
    
\end{lemma}
\begin{proof}
    See \citet[Lemma 35]{efroni2020exploration}.
\end{proof}

\section{Learning in Unknown Constrained Markov Potential Games - Safe Exploration Without a Generative Model (\cref{subsec:cmdp-no-regret})}
\label{app:cmdp-no-regret}

In this section, we consider the setting, where the agents want to explore safely, but they do not have access to a generative model anymore. Existing algorithms with safe exploration \citep{Bura2022, liu2021learning} have guarantees on the {\em regret}, but no sample complexity guarantees. First, we define a no-regret algorithm with safe exploration guarantees as follows:
\begin{definition}[No-regret algorithm with safe exploration]\label{def:no-regret-algorithm}
Consider a fixed number of $T>0$ rounds and an algorithm $\Alg$, which selects a policy $\pi_t\in\Pi$ in every round $t\in [T]$. $\Alg$'s regret after $T$ rounds is defined as follows:
\begin{equation*}
   R(T) \triangleq \sum_{t=1}^T \max_{\pi\in\Pi_C} \V^r(\pi) - \V^r(\pi_t).
\end{equation*}
We call $\Alg$ is called a no-regret algorithm with safe exploration, if $R(T)\in o(T)$ and if $\pi_t\in \Pi_C, \forall t\in [T]$. Examples of such algorithms are \citet{liu2021learning} and \citet{Bura2022}.
\end{definition}
Next, we discuss how we can use the notion of regret to derive a sample complexity bound. For {\em unconstrained} MDPs, sub-linear regret bounds can be converted to a sample complexity bound by applying the well-known {\em online-to-batch} conversion trick \citep{jin2018q}. Applying the same trick to the DOPE algorithm \citet{Bura2022}, we derive a sample-efficient algorithm in \cref{alg:cmdp-no-regret} and prove its sample complexity in the following lemma.
\begin{restatable}{lemma}{lemcmdpscnr}
\label{lem:cmdp-sample-complexity-no-regret}
    Consider a CMDP $\M = \paren{\StateSpace,\Action,H,\curly{\reward_h}_{h\in[H]}, \curly{\Transition_h}_{h\in [H]}, \curly{c_h}_{h\in [H]}, \alpha}$. Assume that a strictly feasible policy $\pi^S$ and its constraint value $\alpha_S\triangleq\V^c(\pi^S)$ are known s.t. $\alpha_S<\alpha$. Then, for any $\epsilon' \in (0,H]$ and $\delta' \in (0,1)$, \cref{alg:cmdp-no-regret} with $T = \O\paren{\frac{|\StateSpace|^2 H^6 |\Action|}{\paren{\alpha-\alpha_S}^2 \epsilon'^2}}$ and $M = \wtilde{\O}\paren{\frac{H^2}{\epsilon'^2} \log\paren{\frac{1}{\epsilon' \delta'}}}$ returns a policy $\hat{\pi}\in\Pi_C$ s.t. $\max_{\pi \in \Pi_C} \V^{r}(\pi) - \V^r(\hat{\pi}) \leq \epsilon$ holds, with probability at least $1-\delta'$. This results in a sample complexity of $\F_{C} \paren{|\StateSpace|,|\Action|,\Horizon,\alpha-\alpha_S,\delta',\epsilon'} = \wtilde{\O}\paren{\frac{|\StateSpace|^2 |\Action| \Horizon^9}{ \paren{\alpha-\alpha_S}^2 \epsilon'^4 } \log\paren{\frac{1}{\epsilon' \delta'}} }$.
\end{restatable}
\begin{proof}

First, recall that \cref{lem:dope-regret} provides the following regret guarantees with probability at least $1-5\delta$:
\begin{align}
    \pi_t &\in \Pi_C, \forall t\in [T], \label{eq:dope-constraint} \\
    R(T) &\leq \wtilde{\O}\paren{\frac{|\StateSpace|\Horizon^3}{\paren{\alpha-\alpha^S}} \sqrt{|\Action| T} }. \label{eq:dope-regret}
\end{align}

\textbf{Constraint violation:} Since $\hat{\pi}$ is selected from $\curly{\pi_t}_{t\in [T]}$ (\cref{algln:no-regret-final-policy}) and \cref{eq:dope-constraint} holds, the policy $\hat{\pi}$ is a feasible policy too.

\textbf{Reward sub-optimality:} Applying \cref{lem:regret-to-sc} with $C = \frac{|\StateSpace| H^3 \sqrt{|\Action|}}{\paren{\alpha-\alpha^S}}$, we obtain that with $T = \O\paren{\frac{|\StateSpace|^2 H^6 |\Action|}{\paren{\alpha-\alpha^S}^2 \epsilon^2}}$ and $M = \wtilde{\O}\paren{\frac{H^2}{\epsilon^2} \log\paren{\frac{1}{\epsilon \delta}}}$, the returned policy $\hat{\pi}$ satisfies $\max_{\pi\in\Pi_C}\V^r(\pi) - \V^r(\hat{\pi}) \leq \epsilon$, with probability at least $1-\delta$. Taking a union bound over this, \cref{eq:dope-constraint} and \cref{eq:dope-regret}, the following guarantees hold with probability at least $1-6\delta$:
\begin{align*}
    \max_{\pi \in \Pi_C}\V^r(\pi) - \V^r(\hat{\pi}) &\leq \epsilon,\\
    \V^c(\hat{\pi}) &\leq \alpha.
\end{align*}
To compute the final sample complexity, note that \Dope~internally uses an initial phase of $T_0$ episodes, during which the agent plays the initial policy $\pi^S$. We do not discuss the details here, but need to account for those $T_0$ episodes in the sample complexity. The resulting sample complexity is as follows:
\begin{align*}
    \F_C\paren{|\StateSpace|, |\Action|, \Horizon, \zeta, \delta, \epsilon} &= \paren{\underbrace{T+T_0}_{\text{\cref{algln:dopa}}} + \underbrace{TM}_{\text{\cref{algln:no-regret-final-policy}}}} H \\
    &= \wtilde{\O} \paren{\frac{|\StateSpace|^2 |\Action| \Horizon^9 }{\paren{\alpha-\alpha^S}^2 \epsilon^4} \log\paren{\frac{1}{\epsilon \delta}}} + \wtilde{\O}\paren{\frac{|\StateSpace|^2 |\Action| \Horizon^5}{\paren{\alpha-\alpha^S}^2}} \tag{$T_0 = \wtilde{\O}\paren{\frac{|\StateSpace|^2 |\Action| \Horizon^4}{\paren{\alpha-\alpha^S}^2}}$ by \cref{lem:dope-regret}.}\\
    &= \wtilde{\O} \paren{\frac{|\StateSpace|^2 |\Action| \Horizon^9 }{\paren{\alpha-\alpha^S}^2 \epsilon^4} \log\paren{\frac{1}{\epsilon \delta}}}.
\end{align*}
\end{proof}

\textbf{Remark:}
Note that we used a specific definition of no-regret algorithms in \cref{def:no-regret-algorithm}. Other notions of no-regret algorithms for CMDPs exist in the literature \citep{efroni2020exploration}, but they usually do not require feasibility of the iterates (safe exploration). The trick in \cref{lem:regret-to-sc} assumes that at least one of the iterates is {\em both} feasible and $\epsilon$-optimal. This may not be guaranteed by no-regret algorithms without the safe exploration guarantee.

\subsection{Proofs of auxiliary lemmas}

\begin{lemma}[Theorem 3 from \citet{Bura2022}]
\label{lem:dope-regret}
Consider a CMDP $\M = \paren{\StateSpace,\Action,H,\curly{\reward_h}_{h\in[H]}, \curly{\Transition_h}_{h\in [H]}, \curly{c_h}_{h\in [H]}, \alpha}$. Fix any $\delta\in(0,1)$. Then, \Dope~invoked with $T_0 = \wtilde{\O}\paren{\frac{|\StateSpace|^2 |\Action| \Horizon^4}{\paren{\alpha-\alpha^S}^2}}$ generates a sequence of policies $\curly{\pi_t}_{t\in [T]}$ s.t. $\pi_t\in\Pi_C, \forall t\in [T]$ and the sequence has the following regret:
\begin{equation*}
    R(T) = \sum_{t=1}^T \max_{\pi\in\Pi_C} \V^r(\pi) - \V^r(\pi_t) \leq \wtilde{\O}\paren{\frac{|\StateSpace|\Horizon^3}{\paren{\alpha-\alpha^S}} \sqrt{|\Action| T} },
\end{equation*}
with probability at least $1-5\delta$.
\end{lemma}

\begin{lemma}[Regret to Sample Complexity]
\label{lem:regret-to-sc}
Consider a no-regret CMDP algorithm $\Alg$ (see \cref{def:no-regret-algorithm}) with regret bound $R(T) = \sum_{t=1}^T \max_{\pi\in\Pi_C} \V^r(\pi) - \V^r(\pi_t) \leq C\sqrt{T}$, where $C$ is a constant that does not depend on $T$, and $\pi_t\in\Pi_C, \forall t\in [T]$. Given $\epsilon\in (0,H], \delta\in (0,1)$, run algorithm $\Alg$ with $T = \frac{4 C^2}{\epsilon^2}$ episodes. Next, execute each policy $\pi \in \curly{\pi_t}_{t\in [T]}$ for $M=\frac{16 H^2}{\epsilon^2} \log \paren{\frac{2 C}{\epsilon \delta}}$ episodes and estimate the value functions $\curly{\hat{\V}^r(\pi_t)}_{t\in [T]}$. Then, the policy $\hat{\pi} = \arg\max_{\pi \in \curly{\pi_t}_{t\in [T]}} \curly{\hat{\V}^r(\pi)}$ satisfies the following guarantees:
\begin{equation*}
    \V^r(\pi^*) - \V^r(\hat{\pi}) \leq \epsilon,
\end{equation*}
with probability at least $1-\delta$.
\end{lemma}
\begin{proof}
    For any $t\in [T]$, we can bound the estimation error of the value function for $\pi_t$ as follows:
    \begin{align*}
        Pr\bracket{\abs{\V^r(\pi_t) - \hat{\V}^r(\pi_t)}\geq \frac{\epsilon}{4}} &\leq 2 \exp \paren{\frac{-2 N \paren{\frac{\epsilon}{4}}^2 }{H^2}} \tag{Hoeffding's inequality.}\\
        &= 2 \exp \paren{\frac{-N \epsilon^2}{8 H^2}} \\
        &= \frac{\delta}{T} \tag{Setting $N=\frac{16 H^2}{\epsilon^2} \log \paren{\frac{2 C}{\epsilon \delta}}$.}
    \end{align*}
    Taking the union bound over all $t\in [T]$, we obtain $Pr\bracket{\exists t\in [T]: \abs{\V^r(\pi_t) - \hat{\V}^r(\pi_t)}\geq \frac{\epsilon}{4}} \leq \delta$.
    
    Next, we note that the average regret can be bounded as $\frac{R(T)}{T} \leq \frac{C}{\sqrt{T}} = \frac{\epsilon}{2}$ with $T=\frac{4 C^2}{\epsilon^2}$. Thus, there exists $t^* \in \arg\min_{t\in[T]}\curly{\V^r(\pi^*)-\V^r(\pi_t)}$ s.t. $\V^r(\pi^*)-\V^r(\pi_{t^*})\leq \frac{R(T)}{T} \leq \frac{\epsilon}{2}$ holds. We do a case analysis on $\V^r(\pi_{t^*}) - \V^r(\pi)\geq 0$:

    \textbf{Case 1:} $\V^r(\pi_{t^*}) - \V^r(\pi) \leq \frac{\epsilon}{2}$. Then, we can bound $\V^r(\pi^*)-\V^r(\pi)$ as follows:
    \begin{align*}
        \V^r(\pi^*) - \V^r(\pi) &= \bracket{\V^r(\pi^*) - \V^r(\pi_{t^*})} + \bracket{\V^r(\pi_{t^*})-\V^r(\pi)} \tag{Addition/subtraction of the common term.} \\
        &\leq \epsilon.
    \end{align*}

    \textbf{Case 2:} $\V^r(\pi_{t^*}) - \V^r(\pi) > \frac{\epsilon}{2}$. In this case, note that the following property must hold for the estimated value functions:
    \begin{align*}
        \hat{\V}^r(\pi_{t^*}) &\geq \V^r(\pi_{t^*}) - \frac{\epsilon}{4} \\
        &> \V^r(\pi) + \frac{\epsilon}{4} \\
        &\geq \hat{\V}^r(\pi).
    \end{align*}
    Thus, in this case, $\pi$ cannot be the maximizer of $\curly{\hat{\V}^r(\pi_t)}_{t\in [T]}$.

    Overall number of episodes required for this technique: $TN = \frac{64 C^2 H^2}{\epsilon^4} \log \paren{\frac{2 C}{\epsilon \delta}}$.
\end{proof}

\end{document}